%% file: the-paper.tex
\documentclass[sigconf]{aamas} 

\usepackage{balance} 
\usepackage{tablefootnote}
\usepackage{enumerate}
\usepackage{longtable}
\usepackage[para]{threeparttable}
\usepackage{array}
\usepackage{tikz}
\usepackage{tabto}
\usepackage{bm}
 \usepackage{relsize}

\newcommand{\qedSmiley}[1]{
\hspace*{\fill}
\begin{tikzpicture}[scale=0.4]
\draw (0,0) -- (0.5,0) -- (0.5,0.5) -- (0,0.5) -- (0,0);
\draw (0.10,0.25) arc (-180:0:0.15);
\draw (0.20,0.25) -- (0.20,0.40);
\draw (0.30,0.25) -- (0.30,0.40);
\end{tikzpicture}}

\newtheorem*{proposition*}{Proposition}

\setcopyright{none}
\acmConference[ALA '22]{Proc.\@ of the Adaptive and Learning Agents Workshop (ALA 2022)}
{May 9-10, 2022}{Online, \url{https://ala2022.github.io/}}{Cruz, Hayes, da Silva, Santos (eds.)}
\copyrightyear{2022}
\acmYear{2022}
\acmDOI{}
\acmPrice{}
\acmISBN{}
\title[Is VPG Overlooked? Analyzing Deep RL for Hanabi]{Is Vanilla Policy Gradient Overlooked?\\Analyzing Deep Reinforcement Learning for Hanabi}

\author{Bram Grooten}
\affiliation{
  \institution{Eindhoven University of Technology}
  \city{Eindhoven}
  \country{Netherlands}}
\email{b.j.grooten@tue.nl}

\author{Jelle Wemmenhove}
\affiliation{
  \institution{Eindhoven University of Technology}
  \city{Eindhoven}
  \country{Netherlands}}
\email{a.j.wemmenhove@tue.nl}

\author{Maurice Poot}
\affiliation{
  \institution{Eindhoven University of Technology}
  \city{Eindhoven}
  \country{Netherlands}}
\email{m.m.poot@tue.nl}

\author{Jim Portegies}
\affiliation{
  \institution{Eindhoven University of Technology}
  \city{Eindhoven}
  \country{Netherlands}}
\email{j.w.portegies@tue.nl}

\begin{abstract}
In pursuit of enhanced multi-agent collaboration, we analyze several on-policy deep reinforcement learning algorithms in the recently published Hanabi benchmark. Our research suggests a perhaps counter-intuitive finding, where Proximal Policy Optimization (PPO) is outperformed by Vanilla Policy Gradient over multiple random seeds in a simplified environment of the multi-agent cooperative card game. In our analysis of this behavior we look into Hanabi-specific metrics and hypothesize a reason for PPO's plateau. In addition, we provide proofs for the maximum length of a perfect game (71 turns) and any game (89 turns).
Our code can be found at: \url{https://github.com/bramgrooten/DeepRL-for-Hanabi}.
\end{abstract}

\keywords{Deep reinforcement learning, Hanabi, Vanilla Policy Gradient, PPO, multi-agent collaboration}

\newcommand{\BibTeX}{\rm B\kern-.05em{\sc i\kern-.025em b}\kern-.08em\TeX}

\begin{document}


\pagestyle{fancy}
\fancyhead{}


\maketitle 


\input{sections/1-intro}

\input{sections/2-related-work}

\input{sections/3-experiments}

\input{sections/4-analysis}

\input{sections/5-conclusion}



\begin{acks}

Thank you to Decebal Constantin Mocanu for his ongoing guidance, and to Qiao Xiao and Mickey Beurskens for reviewing the paper. Also, much graditude goes to Nolan Bard for helping us set up his team's Hanabi Learning Environment.

\end{acks}


\bibliographystyle{ACM-Reference-Format} 
\bibliography{references}

\section*{Appendix}
\appendix

\input{sections/a0-mdps}



\input{sections/c-hyperparams}
\input{sections/e-extra-figs}

\end{document}

%% file: sections/1-intro.tex
\section{Introduction}

Many real world scenarios such as autonomous driving require multi-agent collaboration through partial observability. A new benchmark was recently developed by 
a group of researchers from DeepMind, who coined the Hanabi Challenge as a new frontier for AI \cite{bard}. 
Reinforcement learning approaches that have been applied to this benchmark so far include asynchronous advantage actor-critic (A3C) algorithms \cite{bard}, deep Q-networks (DQNs) \cite{sad}, and search methods \cite{search}.
%
We missed the application of standard on-policy algorithms such as Vanilla Policy Gradient (VPG) and Proximal Policy Optimization (PPO), so we were motivated to discover whether these methods perform well in this new environment. 
We run experiments to compare the algorithms, and analyze the behavior of the agents.
Our main contributions are:
\begin{enumerate}[I.]
    \item We define a simplified version of Hanabi and apply three deep reinforcement learning algorithms to it, with VPG being the unexpected winner.
    \item We analyze the agents' performance through metrics corresponding specifically to Hanabi, and hypothesize why PPO seems to hit a plateau.
    \item We provide proofs for the maximum length of a regular and a perfect Hanabi game, being 89 and 71 turns respectively. The latter number contradicts earlier literature.
\end{enumerate}
We will first explain the rules of Hanabi, after which we go into related work in Section~\ref{sec:related-work}. The setup and results of our experiments are shown in Section~\ref{sec:experiments}. Section~\ref{sec:analysis} analyzes the outcomes while diving into Hanabi-specific properties, such as the game length. Lastly, Section~\ref{sec:conclusion} concludes the paper.


\begin{figure}
    \centering
    \includegraphics[width=\linewidth]{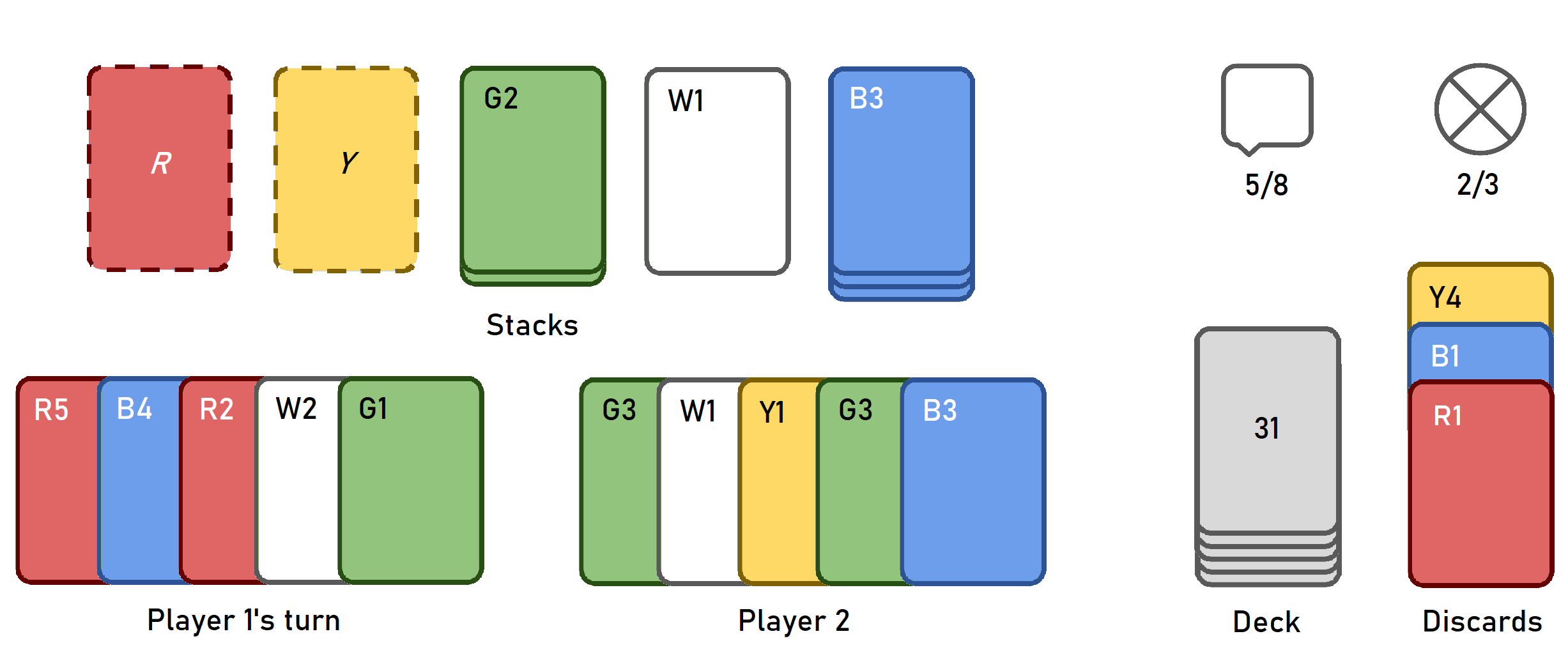}
    \caption{Example of a game state. Image adjusted from \cite{bard}.}
    \label{fig:state}
\end{figure}

\subsection*{Hanabi}

We briefly explain the rules of Hanabi. The card game can be played with 2 to 5 players who collaborate to achieve the highest score possible. The twist is that you cannot see your own cards, but you do see all the other player's cards. By giving each other (restricted) hints players can deal with this imperfect information.

The goal of the game is to form stacks of cards, one for each of the five colors, see Figure~\ref{fig:state}. Every card has a color and a rank between 1 and 5. A stack must begin with a rank 1 card, and build all the way up to 5. If all stacks (also called \textit{fireworks}\footnote{\textit{Hanabi} is actually Japanese for \textit{fireworks}.}) have been completed, the perfect score of 25 has been reached.

Players start with 5 cards in their hand (or 4 when playing with four or five players). 
During her turn, a player may do one of three things: give a hint to another player, play a card, or discard a card. Every time a card is played or discarded, the player gets a new card from the deck.

To give a hint, a player must choose one rank or color and point at all the cards with this property in an other player's hand. This can be done as long as there are hint tokens available, the game starts with just 8. Fortunately, if a player discards a card the group retrieves one hint token.\footnote{Except if there are already 8 hint tokens available, then discarding is not allowed.} Half of the 50 cards in total are duplicates\footnote{To be exact: there are three duplicates of rank 1 cards, two duplicates of cards with ranks 2, 3, or 4, and rank 5 cards are unique.}, so discarding may happen a lot. One hint token is also retrieved when the stack of a certain color is completed.

When a player is confident enough that one of her cards will fit on top of a stack, she can play it. If placed successfully the score goes up by one, otherwise the card will be moved to the discard pile and the group loses one life token. If all 3 life tokens are lost, the game ends and the score goes down to 0. The game also ends if the perfect score has been reached, or when the deck is empty. In the latter case each player gets one more turn, including the one who emptied the deck.

%% file: sections/2-related-work.tex
\section{Related Work}
\label{sec:related-work}


The challenge paper by Bard et al. \cite{bard} served as a starting point for our research. It provides the Hanabi Learning Environment\footnote{See \url{https://github.com/deepmind/hanabi-learning-environment}.} which we build upon in our implementations. Furthermore, they defined two separate research domains called
\textit{self-play} and \textit{ad-hoc}. In self-play an agent only plays with copies of itself, while in ad-hoc agents must be able to play with a wide range of other agents or even human players. 
Most of the current literature focuses on self-play, with a couple of exceptions \cite{canaan2019, eger2017}. Our research also stays in the self-play domain.

Another important distinction is the approach used to program an agent for Hanabi. We separate them into the categories:
with or without machine learning. We call the agents that do not use any learning method \textit{rule-based}, and it turns out that they are still outperforming the learning agents in many cases. In our previous work \cite{thesisBram} we presented an overview of the state-of-the-art of both approaches, which we will briefly summarize and update here.

\subsection{Rule-based agents}

Within the rule-based regime there again exist two categories: bots that are based on human Hanabi conventions \cite{conventions}, and bots that use \textit{hat-guessing} strategies \cite{cox-hat}. Both approaches can achieve quite decent scores in self-play, but not in ad-hoc play. 

The hat-guessing method is based on a mathematical game where players have to guess the color of their own hat. In Hanabi players do not know the color of their own cards, so this called for similar strategies. By using modular arithmetic, a lot of information can be given with a single hint, provided that all players follow the same algorithm. 
The state-of-the-art in self-play (for 3 or more players) is held by a bot that uses this hat-guessing strategy, called WTFWThat \cite{wtfwthat}. Its scores have been improved later on by the use of search methods \cite{search}.
Some of the best bots that use human conventions include SmartBot \cite{smartbot} and FireFlower \cite{fireflower}. 



\subsection{Learning agents}
\label{sec:learning-agents}

In their challenge paper, Bard et al. \cite{bard} apply two existing approaches of deep reinforcement learning to their Hanabi Learning Environment. The Rainbow agent \cite{rainbow} scores an average of about 18.2 out of 25 in self-play,\footnote{Average taken over all possible number of players (2, 3, 4, and 5).} while the Actor-Critic-Hanabi-Agent (ACHA) which Bard et al. based on A3C \cite{acha}, performed better: 20.3 on average. In the ad-hoc domain both agents have scores close to zero.

In 2018 the Bayesian Action Decoder (BAD) \cite{bad} set a record for 2-player games of Hanabi. The next year, Hu \& Foerster improved the bot with the Simplified Action Decoder (SAD) \cite{sad}, which drastically increased the scores among learned policies in self-play for any number of players. The state-of-the-art for 3 to 5 players is still held by the rule-based bot WTFWThat \cite{wtfwthat}, but reinforcement learning is ahead in the 2-player domain, see \autoref{tab:sota}.

The SAD agent provided a simple, yet elegant solution to the problem of updating beliefs during the exploration phase. In this phase many random actions are taken, which can give misleading information about the state of the game to other agents. Thus, only during training, the agents were allowed to communicate their preferred action, while performing a different random action. This simplified the Bayesian reasoning process.

The scores of SAD were further improved through the tabular search method SPARTA \cite{search}. The agents start off with a blueprint policy, which can be any strategy, also a learned one. In every step of the game, the agents perform a search for the best action using many Monte Carlo rollouts. This action can deviate from the blueprint policy. To make sure that the other agents do not misinterpret the action taken, all agents redo the search of every other agent themselves, using the same random seed (which is shared before the game starts). Agents now know whether an action came from the blueprint policy or from search. This improved the state-of-the-art in self-play for every number of players.
Just last year, the same research group increased the 2-player score slightly further by a more efficient search method called RL Search \cite{rlsearch}.

The popular on-policy deep reinforcement learning algorithm PPO had not been applied to Hanabi yet until last year, when Yu et al. \cite{mappo} adjusted the method to MAPPO (Multi-Agent PPO) to make it more applicable to cooperative games. In the 2-player self-play domain their scores are comparable to, but slightly lower than the state-of-the-art. We use the standard, single-agent version of PPO in this research.

\begin{table}[]
\centering
\caption{The state-of-the-art Hanabi agent in self-play for each number of players, to the best of our knowledge. The names in parentheses indicate that these agents have been improved by search methods (RL Search \cite{rlsearch}, SPARTA \cite{search}), which increased their original scores. The table includes average scores $\pm$ standard error of the mean, and the percentage of perfect games. Data is taken from \cite{rlsearch, search}.}
\label{tab:sota}
\begin{tabular}{ccc}
\# Players & Agent   & Score \\ \hline \rule{0pt}{2.3ex}%
2         & Q-learning(+RL Search)      & \begin{tabular}[c]{@{}c@{}}24.62 $\pm$ 0.03\\ 75.9\%\end{tabular}  \\
3         & WTFWThat(+SPARTA) & \begin{tabular}[c]{@{}c@{}}24.83 $\pm$ 0.006\\ 85.9\%\end{tabular} \\
4         & WTFWThat(+SPARTA) & \begin{tabular}[c]{@{}c@{}}24.96 $\pm$ 0.003\\ 96.4\%\end{tabular} \\
5         & WTFWThat(+SPARTA) & \begin{tabular}[c]{@{}c@{}}24.94 $\pm$ 0.004\\ 95.5\%\end{tabular}
\end{tabular}
\vspace{-10pt}
\end{table}


%% file: sections/3-experiments.tex
\section{Experiments}
\label{sec:experiments}

Reading through the literature on Hanabi, we missed the application of the 
standard, or `vanilla', policy gradient algorithm. We were motivated to discover how well it would perform on the new Hanabi benchmark against PPO \cite{ppo}, so we setup a few experiments.

\subsection{Setup}

We compare PPO\footnote{Specifically, PPO-Clip.} with the actor-critic algorithm Vanilla Policy Gradient (VPG), as well as an even simpler algorithm that only has a policy network (actor) and no value network (critic), which we call Simple Policy Gradient (SPG). Both SPG and VPG are based on the classic REINFORCE algorithm \cite{williams92}. Our implementations are built upon the SpinningUp documentation by OpenAI \cite{SpinningUp2018}. 

We use the simplified or `cheat' version of Hanabi, which means that players \textit{are} now allowed to view their own cards. This greatly reduces the complexity of the game, although it has been proven that the problem of finding a winning play sequence is still NP-complete in this case \cite{NP-complete}. If we consider the closed deck of cards to be part of the transition function of a Markov decision process (MDP) instead of being part of the state, then the game has now become fully observable. This can be done by viewing the shuffled deck as a uniformly random distribution over all cards that are left. We have reduced the decentralized partially observable MDP (Dec-POMDP) of Hanabi to a multi-agent or decentralized MDP (MMDP or Dec-MDP). See Appendix~\ref{sec:mdps} for an overview of the different mathematical frameworks.

For us it means that searching through the action space becomes much more manageable, as our policy networks only need 11 output neurons. We stick to the two-player version of Hanabi, so each player has 5 cards it can play or discard, giving 10 actions. We include one more action neuron which produces a random hint when selected. Sharing information in simplified Hanabi is superfluous, but the action is still necessary to lower the hint token budget such that discarding is allowed.\footnote{The random hint action can also be used to `pass' the turn to the other player.}

We will now go into some implementation details. 
Our policy network and value network both receive the state of the game as input, which is encoded into a binary vector of length 136 in the following way. 
First, the firework stacks are represented in thermometer style, with five binary numbers for each color. For example, $[1,1,1,0,0]$ means that the firework of a certain color is at rank 3. 
For each of the player's own cards we include a one-hot encoding for the color as well as the rank. The Y4 card for instance is represented by the piece $[0,1,0,0,0,\ \  0,0,0,1,0]$.
The discard pile is included with 10 binary values per color, then grouped by rank. Thus, $[1,1,0,\ \ 0,0,\ \ 1,0,\ \ 0,0,\ \ 0]$ means that two rank 1 cards and one rank 3 card of a certain color have been discarded. 
Lastly, the vector pieces $[1,1,0]$ and $[1,1,1,1,1,0,0,0]$ indicate that there are 2 life tokens and 5 hint tokens left.
The total length of this state encoding becomes $5\cdot5 + 5\cdot10 + 5\cdot10 + 3 + 8 = 136$.

All networks have 3 hidden layers of differing sizes, with $\textrm{Tanh}$ activation functions in between. 
The policy network outputs a probability distribution over the 11 possible actions through a softmax activation. The value network (VPG and PPO only) has an output layer with a single neuron and no activation, to be able to estimate the true state value function of the current policy: $v^{\pi_{\bm{\theta}}}(s)$.
We use the Adam optimizer with a learning rate of $3 \cdot 10^{-4}$ for both networks. 
The loss function for the value network is mean squared error, while the objective function for the policy network depends on the algorithm. For SPG we use:
\begin{equation*}
    \mathbb{E}_{\pi} \left[ \sum_{t=0}^{\infty} \nabla_{\bm{\theta}} \log \pi_{\bm{\theta}}(A_t\,|\,S_t) \cdot q^{\pi_{\bm{\theta}}}(S_t,A_t) \right]
\end{equation*}
as the policy gradient\footnote{We are aware that Nota and Thomas \cite{nota2020policy} have proven this expression to be incorrect, as it should include a discount factor term: $\gamma^t$. The expression has however been used successfully in practice, so we stick with it.}. Here $\pi_{\bm{\theta}}(a\ |\ s)$ denotes the probability of selecting action $a$ in state $s$ with our current policy $\pi$ parameterized by $\bm{\theta}$. Capital letters stand for random variables. Lastly, $q^{\pi_{\bm{\theta}}}(s,a)$ is the true state-action value function of the current policy, which the algorithm estimates by running about 10 episodes of Hanabi. 
For VPG we have:
\begin{equation*}
   \mathbb{E}_{\pi} \left[ \sum_{t=0}^{\infty} \nabla_{\bm{\theta}} \log \pi_{\bm{\theta}}(A_t\,|\,S_t) \cdot A_t^{\pi_{\bm{\theta}}} \right] 
\end{equation*}
where $A_t^{\pi_{\bm{\theta}}}$, the advantage function, is defined as 
\begin{equation*}
    A_t^{\pi_{\bm{\theta}}} = A^{\pi_{\bm{\theta}}}(S_t, A_t) = q^{\pi_{\bm{\theta}}}(S_t,A_t) - v^{\pi_{\bm{\theta}}}(S_t).
\end{equation*}
We use generalized advantage estimation (GAE) \cite{gae} to approximate this quantity.
For PPO a totally different expression is maximized:
\begin{equation*}
    \mathbb{E}_{\pi}\bigg[\min \left(
    r_t(\bm{\theta}) A_t^{\pi_{\bm{\theta}_{\text{old}}}}\, \mathlarger{\mathlarger{\bm{,}}} \ 
    \textsc{clip}\left(r_t(\bm{\theta}), 1-\varepsilon, 1+\varepsilon \right) A_t^{\pi_{\bm{\theta}_{\text{old}}}}
    \right)\bigg]
\end{equation*}
where  
\begin{align*}
    r_t(\bm{\theta}) = \frac{\pi_{\bm{\theta}}(A_t\,|\,S_t)}{\pi_{\bm{\theta}_{\text{old}}}(A_t\,|\,S_t)}
    & \text{\ \ \ \ and} & 
    \textsc{clip}(x, a, b) = 
    \begin{cases}
    b & \text{if } x > b \\
    x & \text{if } a \leq x \leq b \\
    a & \text{if } x < a
    \end{cases}. 
\end{align*}
We update the policy network five times per epoch in PPO, such that the clipping operation has effect.\footnote{In the first update iteration of each epoch, we have $\pi_{\bm{\theta}} = \pi_{\bm{\theta}_{\text{old}}}$ so no clipping occurs.}
Each epoch collects a batch of about 1000 environment steps. However, we let the last episode of a batch finish so every epoch has slightly more than 1000 environment steps.
To make sure that our policies play at least 10 episodes per epoch we determined the maximum length of a Hanabi game, shown in Section~\ref{sec:game-length}.

We performed a small hyperparameter search by trying out different network sizes, state encodings, and reward shaping systems. The final settings that we used in our experiments can be found in Table~\ref{tab:hyperparams} of Appendix~\ref{sec:hyperparams}.


\subsection{Results}

\begin{figure}
    \centering
    \includegraphics[width=\linewidth]{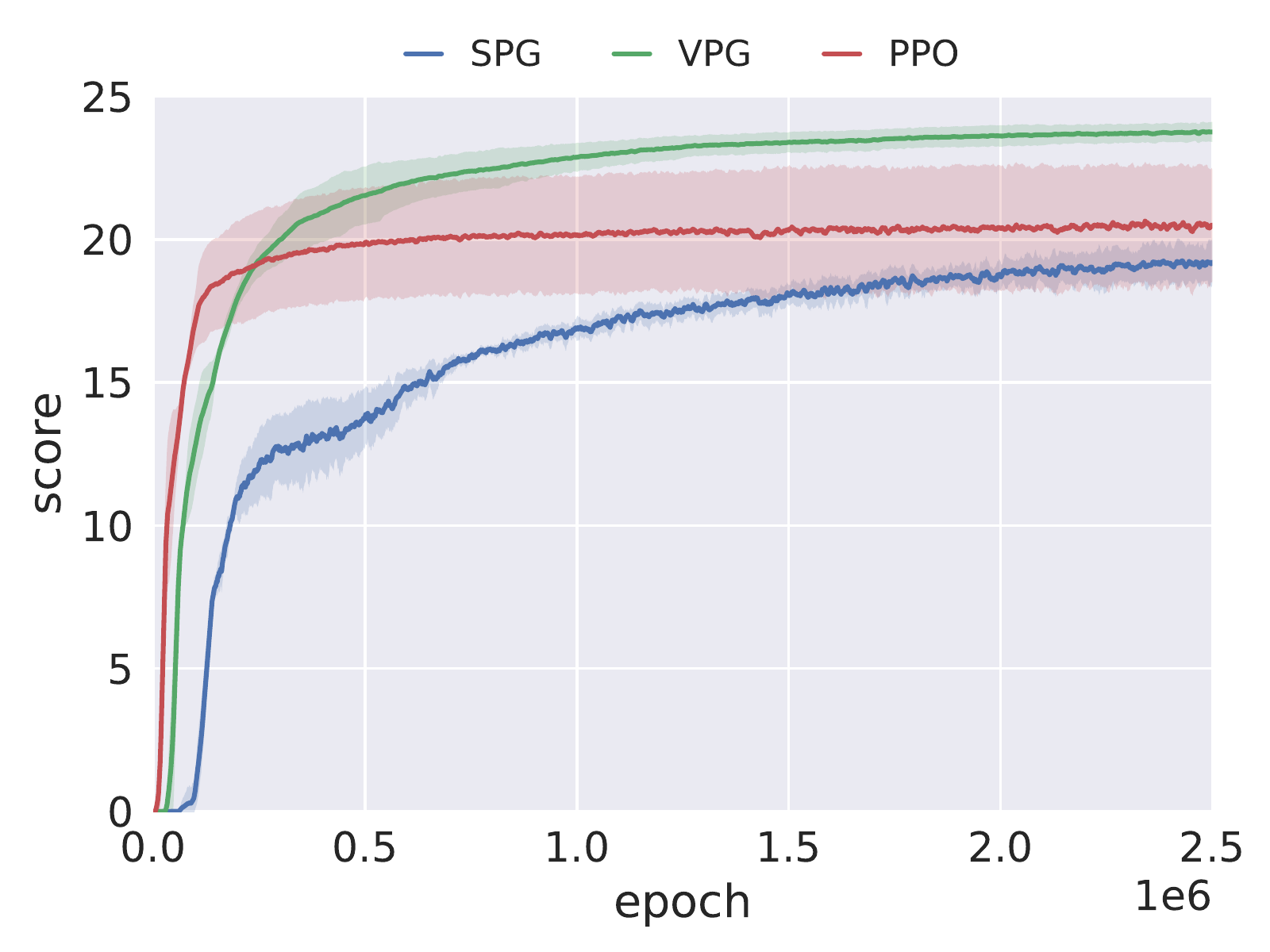}
    \caption{Complete learning curves of the algorithms. The curves show the average of 5 distinct random seeds, with the standard deviation faded above and below. A plot with a separate line for each seed is shown in Figure~\ref{fig:separate-plot} of Appendix~\ref{sec:extra-figs}.}
    \label{fig:comparing}
\end{figure}

To our surprise we notice that PPO is not able to beat the performance of VPG, as shown in Figure~\ref{fig:comparing}. The five runs of PPO all hit a plateau at different levels around an average score just above 20, instead of increasing towards the perfect score of 25. VPG continues to increase slowly and reaches an average score of 23.72 after 2.5 million epochs. Even SPG nears the performance of PPO eventually, albeit at a much slower learning pace.

One of the advantages of PPO in our experiments is that it learns much quicker in the beginning. In Figure~\ref{fig:250K} we again show the learning curves, but only until 250,000 epochs. It takes VPG and SPG quite a lot longer to learn how to increase the score above 0. We noticed that in Hanabi this means an agent needs to learn how to retain at least one life token. In Section~\ref{sec:analysis} we analyze this behavior further.

\begin{figure}
    \centering
    \includegraphics[width=\linewidth]{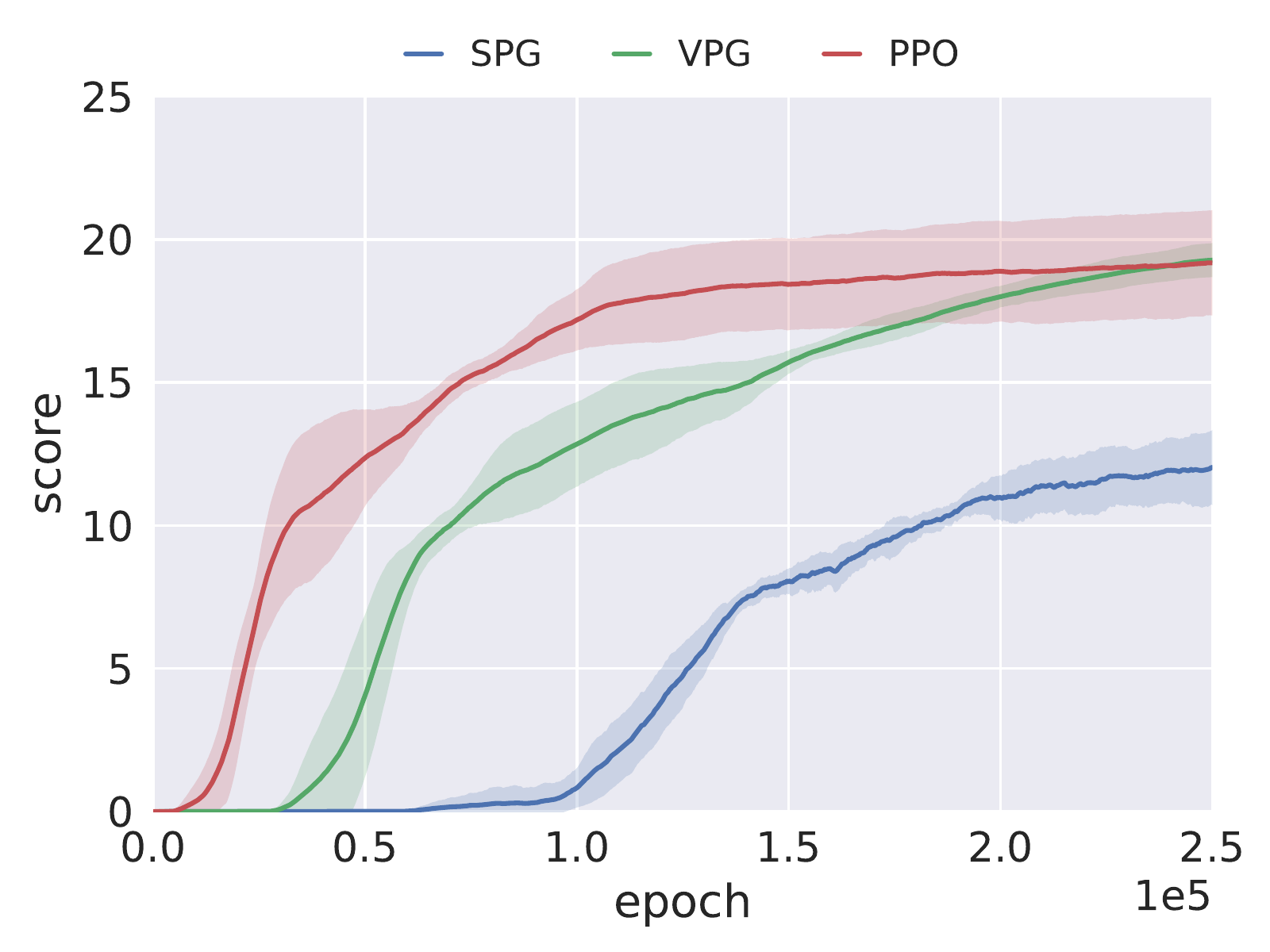}
    \caption{Comparing the scores after just 10\% of training. PPO is the quickest to learn how to increase the score above 0, but is surpassed by VPG later on.}
    \label{fig:250K}
\end{figure}

We tested the final algorithms for 1000 episodes per random seed. The results of these games are shown in Figure~\ref{fig:histo} and Table~\ref{tab:tests}. VPG scores 44.5\% perfect games, while PPO reaches only 13.5\%. Notice that all algorithms still have some failed games of zero points, although VPG almost eliminated them.
In simplified Hanabi it should be much easier to reach a perfect score. Note however, such a winning play sequence does not always exist\footnote{Imagine the situation where all the rank 1 cards are on the bottom of the deck.} \cite{bergh} so we cannot expect a 100\% perfect game proportion.

\begin{figure}
    \centering
    \includegraphics[width=\linewidth]{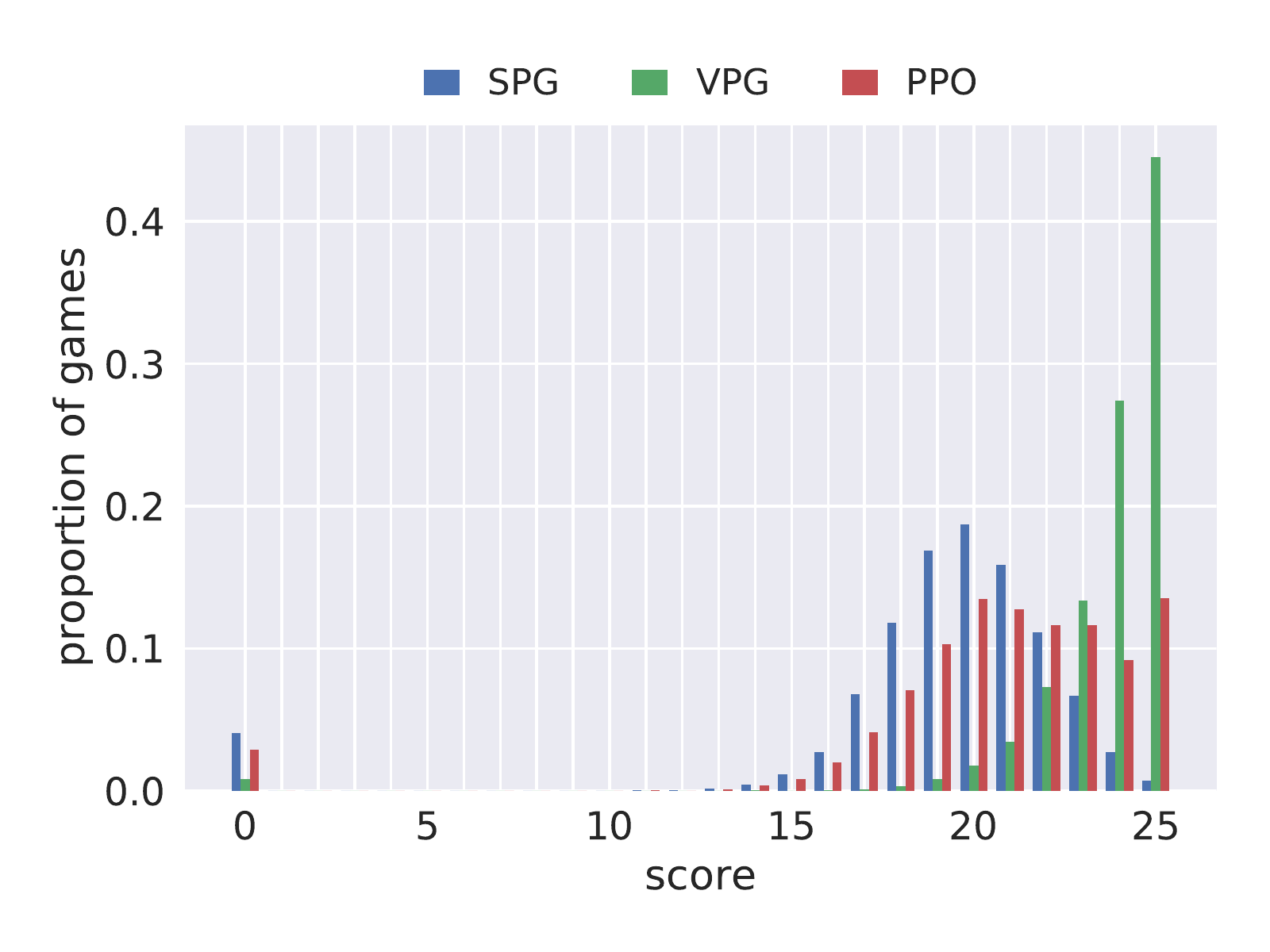}
    \caption{Testing 5000 games per algorithm (1000 for each random seed) after 2.5 million epochs of training.}
    \label{fig:histo}
\end{figure}

\begin{table}[]
\centering
\caption{Performance metrics of 5000 test games after 2.5 million epochs of training. The table includes average scores $\pm$ standard error of the mean, and the percentage of perfect games. The environment is 2-player simplified Hanabi in the self-play domain.}
\label{tab:tests}
\begin{tabular}{ccc}  
                 SPG   & VPG     & PPO \\ \hline \rule{0pt}{2.3ex}%
\begin{tabular}[c]{@{}c@{}}19.09 $\pm$ 0.06\\ 0.7\%\end{tabular}
&
\begin{tabular}[c]{@{}c@{}}23.72 $\pm$ 0.04\\ 44.5\%\end{tabular}
&
\begin{tabular}[c]{@{}c@{}}20.66 $\pm$ 0.06\\ 13.5\%\end{tabular} 
\end{tabular}
\end{table}


%% file: sections/4-analysis.tex
\section{Analysis}
\label{sec:analysis}

In this Section we inspect the performance of our algorithms by looking into a few interesting metrics that we recorded during training, 
such that we can 
hypothesize why PPO scored worse than VPG. We also analyze Hanabi specifically by providing proofs of the maximum length of a regular and a perfect game.

\subsection{Performance analysis}

The following metrics give a better impression of how our algorithms are learning. Let us discuss them one by one.

\textbf{Life tokens.} We keep track of how many life tokens were left over at the end of an episode during training. As we know from the rules of Hanabi, the score decreases back to 0 once all three life tokens are lost.\footnote{There is a variant of Hanabi where the score remains the same, but we do not use it.}
As we see in Figure~\ref{fig:250K}, the algorithms need some time to learn that at least one life token should be left over to maintain its score. In Figure~\ref{fig:life} we see that the number of life tokens shoots up at the same time as the scores go up. Later on, the networks learn that it is not necessary to retain many life tokens, just one is enough. See Appendix~\ref{sec:extra-figs} for figures in this section where the epoch axis goes until the end of training.

\begin{figure}
    \centering
    \includegraphics[width=\linewidth]{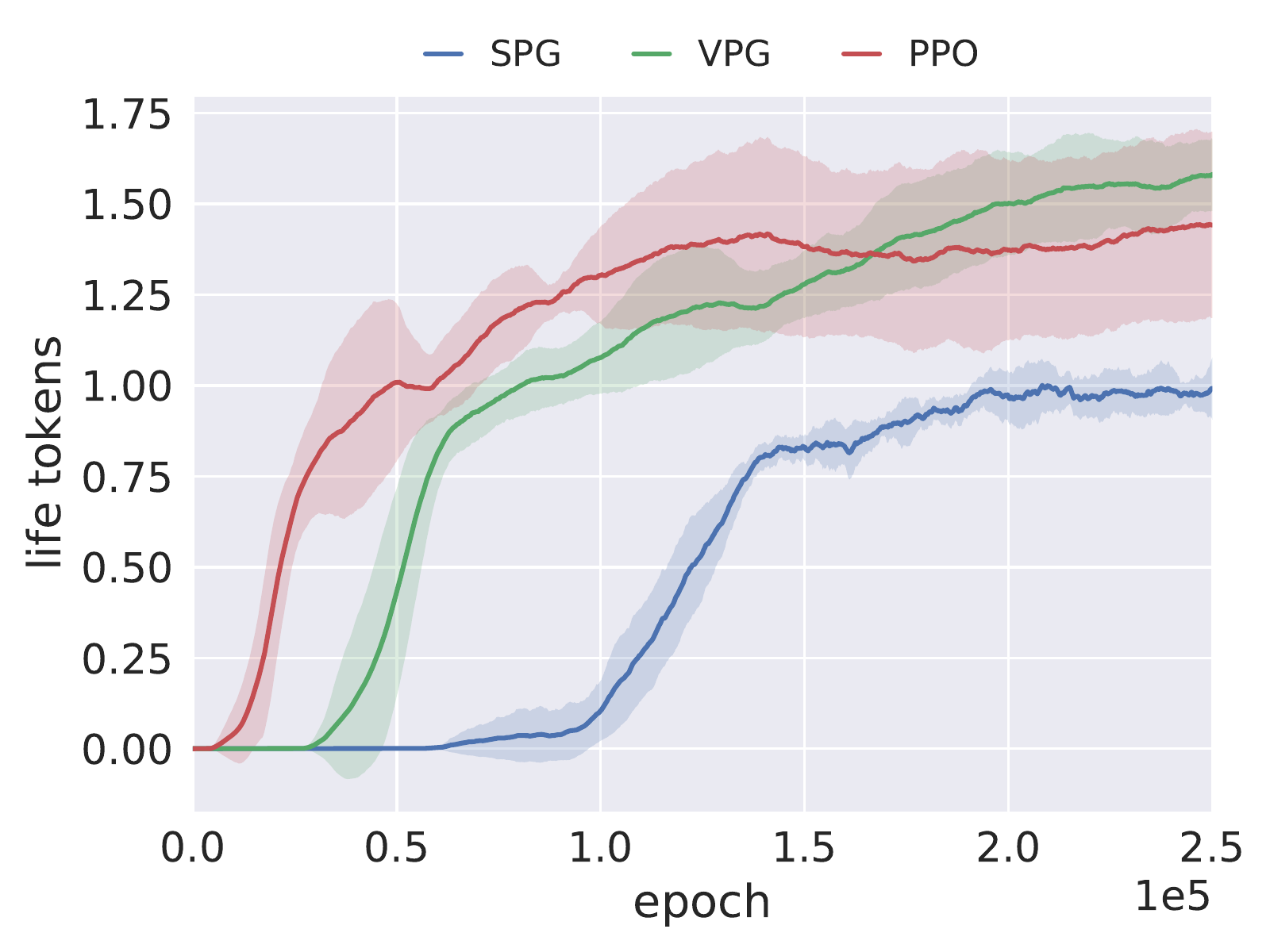}
    \caption{Average number of life tokens left at the end of an episode, shown for the first 10\% of training. The algorithms quickly discover that at least one is needed to get a positive score.}
    \label{fig:life}
\end{figure}

\textbf{Fireworks.} To enable us to see whether an algorithm is actually making progress in this first phase, where all life tokens are constantly lost, we have to look at a different metric than the score. We define the fireworks as the total number of successfully played cards at the end of an episode. If the agents retain at least one life token, then this value equals the score. But when they do not, we can still view their progress with the fireworks metric.

\begin{figure}
    \centering
    \includegraphics[width=\linewidth]{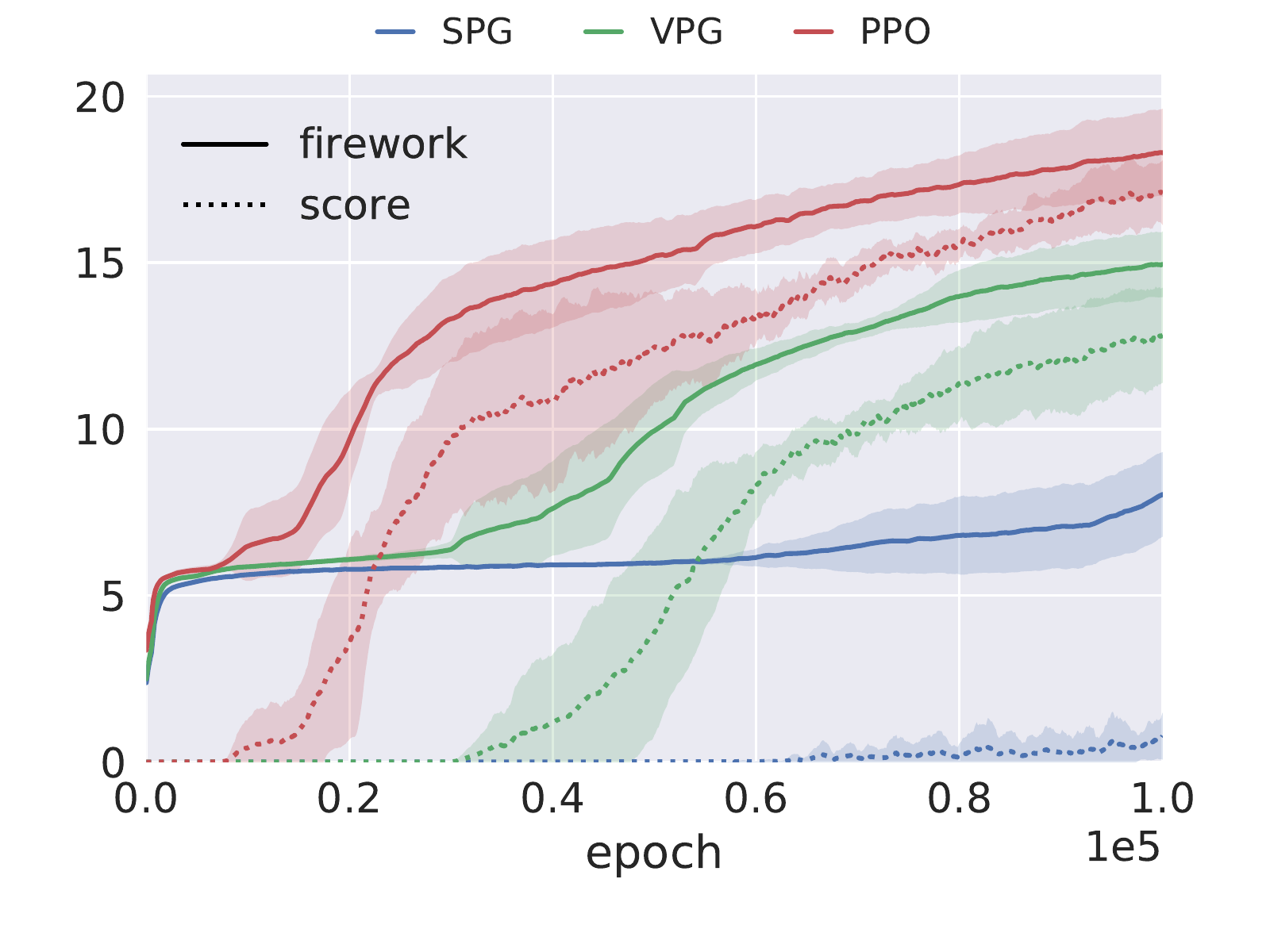}
    \caption{Development of fireworks and scores during training. We stop the graph after just 100,000 epochs for clarity.}
    \label{fig:fireworks}
\end{figure}

In Figure~\ref{fig:fireworks} we show that our agents are actually learning to play cards successfully before they start to retain some life tokens. A fireworks value just above 5 is quickly reached by all three algorithms. We think this is because five rank 1 cards can be played immediately, as long as they have five distinct colors.
Playing higher ranked cards is more difficult. You must meet the extra restriction that a card with the prior rank should be on the stack already.

For each algorithm we see that the moment when the fireworks start to increase far above 5 is simultaneous with the moment that scores go above 0 (life tokens are retained). It seems that in Hanabi learning how to play cards with a rank higher than 1 is the same skill as learning how to retain life tokens, which corresponds to our intuition. 

\textbf{Action probabilities.} To view the development of the action selection probabilities of each agent throughout their training process, we keep track of the average output of the policy networks. In Figure~\ref{fig:actions} we combined the 5 play actions into one category, and did the same with the 5 discard actions. In the very beginning the agents play a lot of (bad) cards, losing all of their life tokens, while after 50,000 epochs the probabilities have completely switched. The agent becomes `scared' to play a card, wanting to retain life tokens. Later on the probability of playing starts to increase again, eventually becoming the preferred action, see Figure~\ref{fig:actions-vpg-max} in Appendix~\ref{sec:extra-figs}.
Increasing the probability of playing cards is something we encouraged with reward shaping. See Appendix~\ref{sec:hyperparams} for our specific rewards.

\begin{figure}
    \centering
    \includegraphics[width=\linewidth]{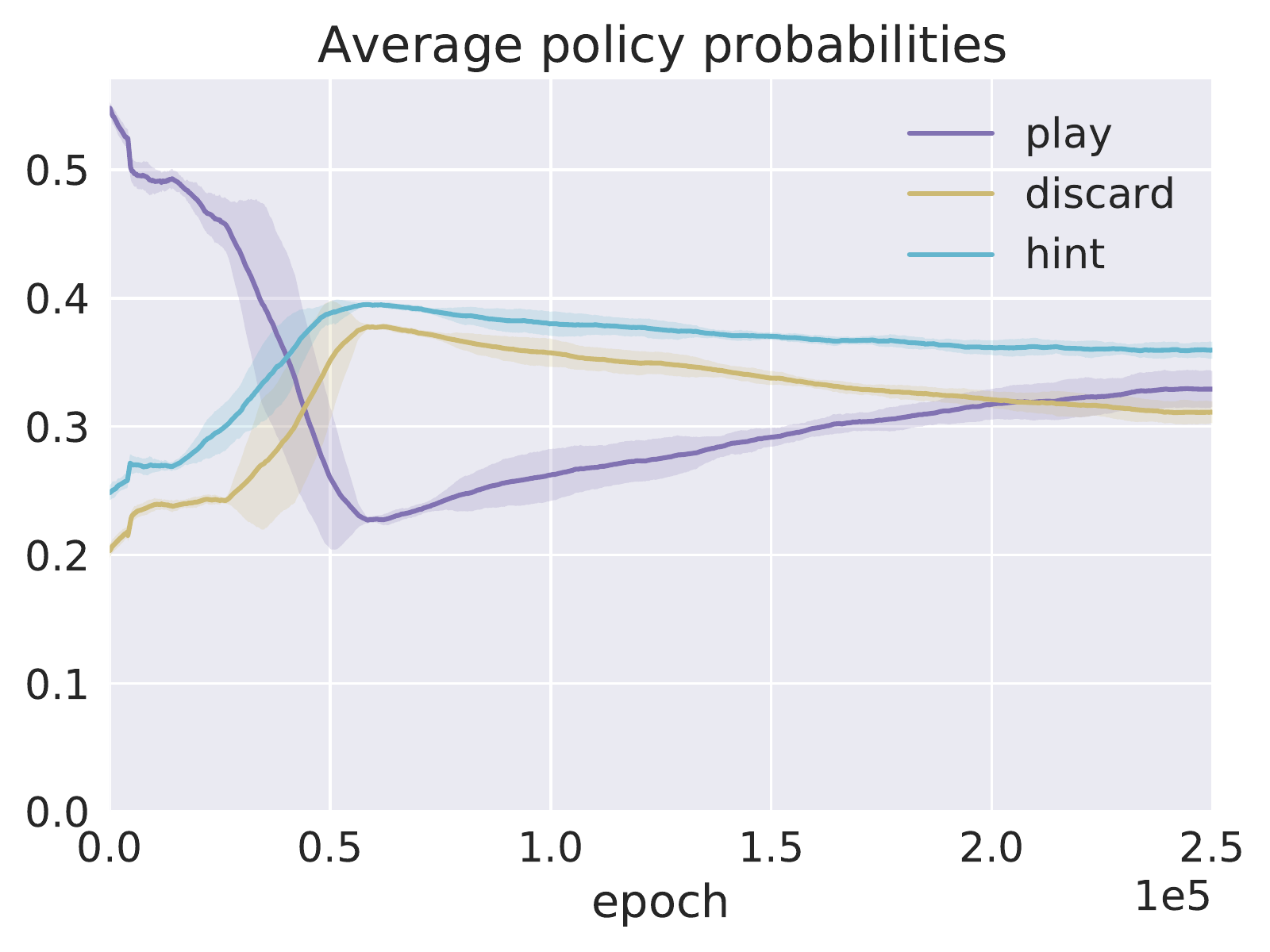}
    \caption{Average policy of our five VPG agents during the first 10\% of training. Similar graphs for SPG and PPO are shown in Appendix~\ref{sec:extra-figs}.}
    \label{fig:actions}
\end{figure}

\textbf{Positional bias.} We want to see whether all card positions in an agents hand are used equally often. For this we plot a histogram representing the policy of one of our VPG\footnote{Histograms for SPG and PPO are in Appendix~\ref{sec:extra-figs}.} agents in Figure~\ref{fig:probs-histo}. It is visible that this agent has a substantial bias towards playing from card position 4, which is the newest card. For simplified Hanabi this makes sense: if you receive a playable card, why not play it immediately?

\begin{figure}
    \centering
    \includegraphics[width=\linewidth]{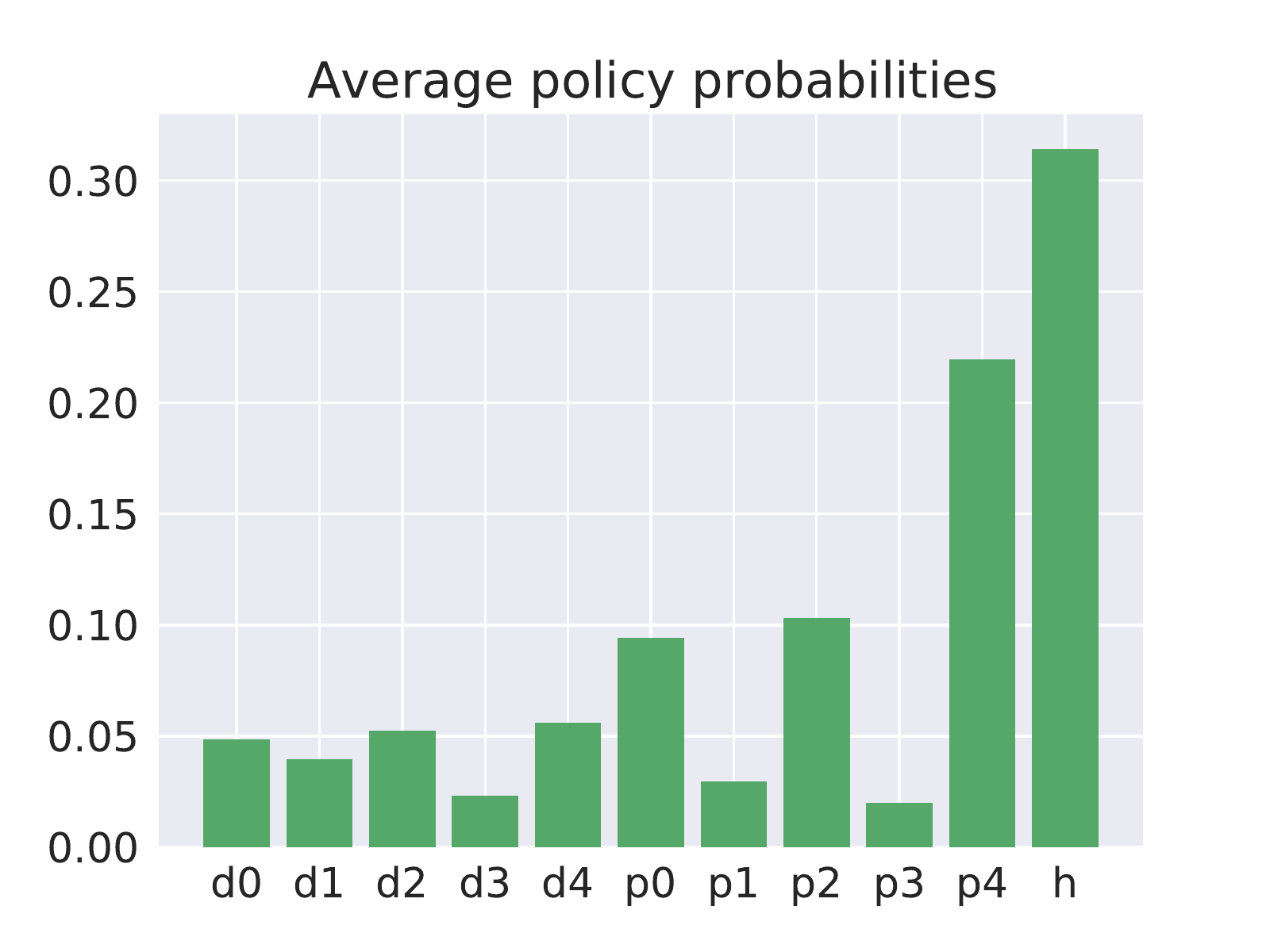}
    \caption{Average action selection probabilities of one VPG run during the last epoch (about 1000 actions). The labels are: $d$ for discard, $p$ for play, and $h$ for giving a random hint. The numbers next to $d$ or $p$ indicate from which index (position in the agent's hand) a card is chosen for that action. New cards always enter the hand at index 4, other cards slide to the left (one index lower) if necessary.}
    \label{fig:probs-histo}
\end{figure}

We want to quantify this positional bias such that we can compare the algorithms. The value should track how large the difference is in the policy's preference for a particular card position relative to the others. 
We define the positional bias as:
\begin{equation*}
    b_{g} = \dfrac{  \max_{i,j \in \mathcal{A}_{g}}\Big( \big| p_i - p_j \big|  \Big) }{ \sum_{k \in \mathcal{A}_{g}} p_k }
\end{equation*}
where $g$ can refer to any subset of actions $\mathcal{A}_{g} \subseteq \mathcal{A}$ and $p_i$ is the average probability of selecting action $i$ under policy $\pi$ given the visited states of the current batch: $p_i = \frac{1}{|B|} \sum_{s \in B} \pi( i | s)$. We track the positional bias of two subsets: the five play actions and the five discard actions.

In words, the positional bias is the greatest distance between two action probabilities within the same subset of actions. On top of that, we rescale this distance to a probability distribution on this specific subset of actions only, to be able to fairly compare the play bias with the discard bias, even if for example the agent discards much more than it plays. The positional bias can take on values between 0 and 1; 0 if the probabilities are all equal, 1 if all the probability mass is on one action. 


The different values of our agents are given in Table~\ref{tab:bias} and plotted throughout training in Figure~\ref{fig:bias-play}.
We see that the best performing agent, VPG, has the lowest positional bias in both categories. Also noteworthy: the play bias is higher than the discard bias for both policy gradient algorithms. Apparently these agents spread out their discard actions more than their play actions, while PPO does not.

\begin{figure}
    \centering
    \includegraphics[width=\linewidth]{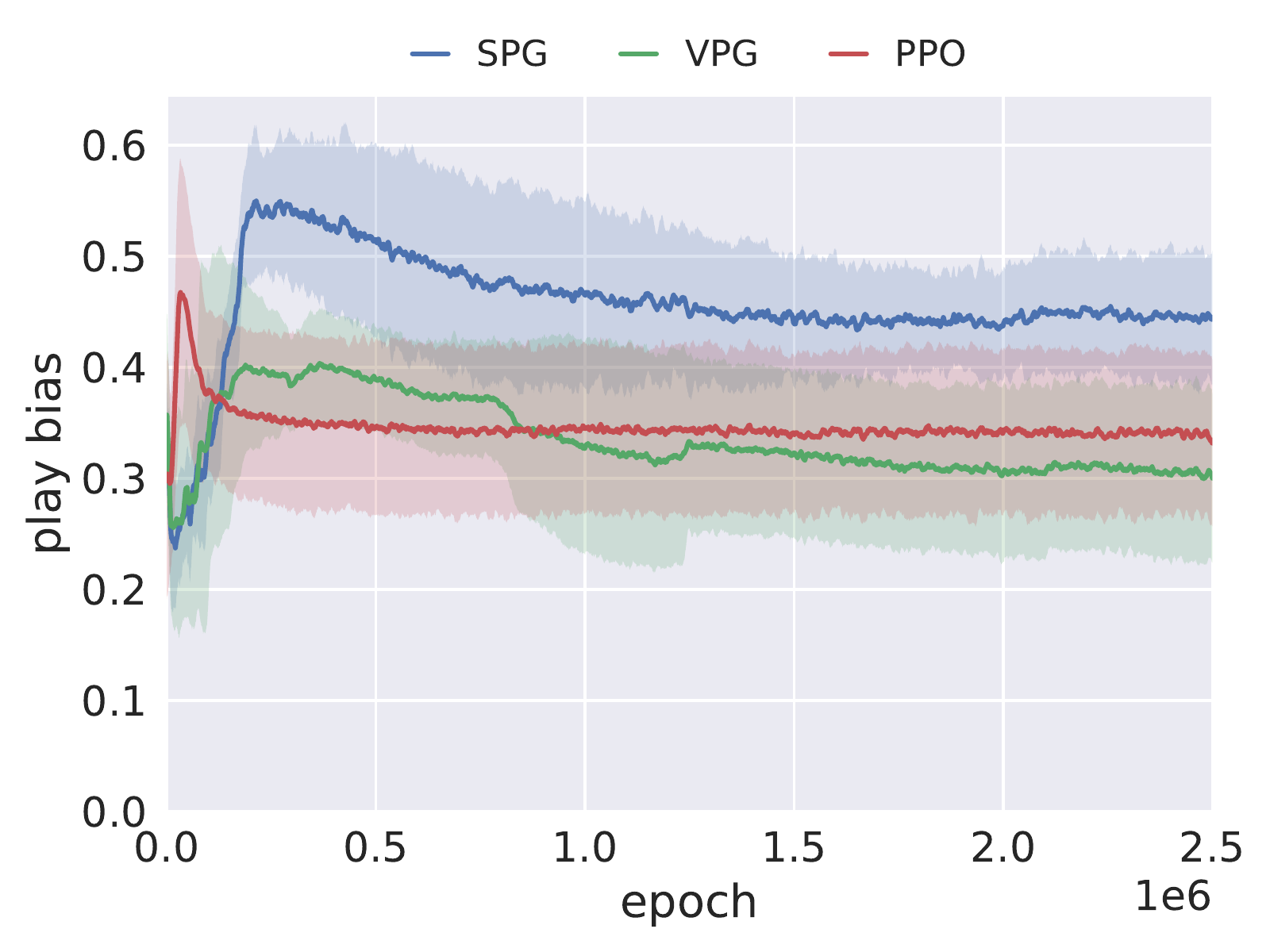}
    \caption{Positional bias of the play actions during training. The development of discard bias is shown in Appendix~\ref{sec:extra-figs}.}
    \label{fig:bias-play}
\end{figure}

At first thought it might seem best to minimize this bias. However, when looking at human play, a certain positional bias is often present as well, for example when applying the popular `chop' convention \cite{conventions}. It says that if you choose to discard, always discard your oldest card which has not received any hints. Unfortunately we cannot say whether we noticed this behavior in our agents, because in our simplified version of Hanabi the agents only give random hints. It would be interesting to see if state-of-the-art Hanabi agents have a high or low positional bias. Our hypothesis is that there is a substantial positional bias, given that for example the BAD\footnote{For Bayesian Action Decoder, see Section~\ref{sec:related-work}.} agent seems to play quite human-like according to their anecdotal analysis \cite{bad}.

\begin{table}[]
\caption{Average positional bias of our agents after 2.5 million epochs of training.}
\label{tab:bias}
\begin{tabular}{c|cc}
    & play bias & discard bias \\ \hline \rule{0pt}{2.3ex}%
SPG & 0.44      & 0.22      \\
VPG & 0.31      & 0.16      \\
PPO & 0.33      & 0.36              
\end{tabular}
\end{table}

\textbf{Entropy.} In some of our preliminary experiments we noticed that our agent converged towards near-deterministic policies rather quickly, even though these policies did not perform well yet. To stimulate more exploration we included an entropy term in the objective function of our policies for all three algorithms, as is regularly done in reinforcement learning and also mentioned in the PPO paper \cite{ppo}.
The new objective function that our policy network's optimizer tries to maximize becomes:
\begin{equation*}
        J_{\text{new}}(\pi_{\bm{\theta}}) = J_{\text{old}}(\pi_{\bm{\theta}}) + \beta \cdot \mathbb{E}_{\pi}[ H_{\pi_{\bm{\theta}}}(S_t) ]
\end{equation*}
where $H$ denotes the information theoretic definition of entropy:
\begin{equation*}
    H_{\pi_{\bm{\theta}}}(s) = - \sum_{a \in \mathcal{A}} \pi_{\bm{\theta}}(a | s) \log \pi_{\bm{\theta}}(a | s)
\end{equation*}
and $\beta$ is the entropy coefficient, which we set to $0.01$ after some fine-tuning.

{
}

During training we kept track of the average entropy of our policies, shown in Figure~\ref{fig:entropy}.
It is noteworthy that PPO has the lowest entropy of all, but does not outperform the others. Our hypothesis is that it has a tendency to get stuck in local minima due to its clipping operation. By clipping the gradient's elements for some parameters, PPO perhaps limits its own learning potential.

\begin{figure}
    \centering
    \includegraphics[width=\linewidth]{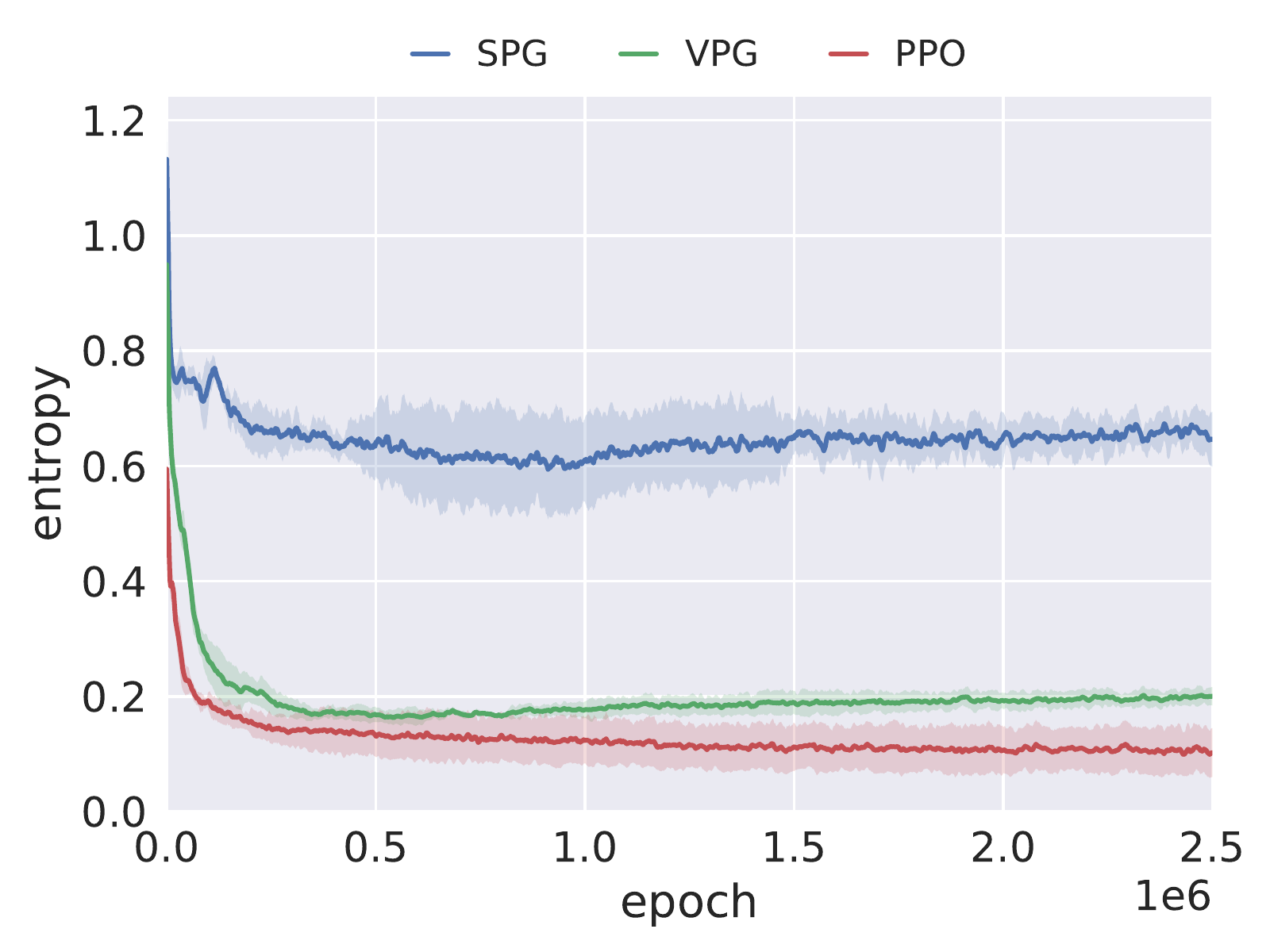}
    \caption{Average entropy of the policies during training.}
    \label{fig:entropy}
\end{figure}

\subsection{Game length}\label{sec:game-length}

To have our algorithms play at least 10 episodes per epoch, we needed to know the maximum length of a Hanabi game. 
We notice that our algorithms take an average of 64.1 steps to complete a Hanabi game at the end of training. 
The maximum length out of the final 1000 test games that each of our agents played was 72 turns.
The real maximum is actually quite a bit higher.

\begin{proposition}\label{89}
The maximum length of a Hanabi game is 89 turns.
\end{proposition}

\begin{proof}
This proof consists of two parts. First we will show that there exists a Hanabi game of length 89. In Part 2
we prove that no Hanabi game can have a higher number of turns than 89.

\textit{Part 1.}
Take a Hanabi game of two players. At the start, each player has 5 cards so there are 40 cards left in the deck. Suppose the players start the game by giving hints until all information tokens are gone. This takes 8 turns.
Then they start a pattern by alternating one discard action and one hint action, continuing until the deck is empty. After the last discard action (which empties the deck) there have been 40 discard actions, with 39 hints in between. Each player gets one more turn, in which they could discard another card. This gives a total of $8 + 40 + 39 + 2 = 89$ turns. 

\textit{Part 2.}
In this part we define a value $\Sigma_t$ for a Hanabi game. We will show that it is impossible for this value to increase during the game ($\Sigma_{t+1} \leq \Sigma_{t} \ \forall t$) from which the maximum number of turns follows. 
We first define a few values: 

\begin{center}
    \begin{tabular}{cl}
        $t$ & total number of turns taken  \\
        $d_t$ & deck size after turn $t$ \\
        $m_t$ & hint tokens left over after turn $t$ \\ 
    \end{tabular}
\end{center}
We further define $c_t$, which stands for the number of hint tokens left over after turn $t$, but with the restriction that these tokens can still be used before the deck is empty:
\begin{equation}\label{ctdef}
    c_t = \begin{cases}
    m_t & \text{ if } d_t > 0,\\
    0 & \text{ if } d_t = 0.
    \end{cases}
\end{equation}
We add the restriction to $c_t$ here to distinguish between the situations before and after the deck has been emptied. Once the deck is empty, hint actions cannot be used to stall the game anymore. When $d_t = 0$, there is a fixed maximum number of turns left, which we denote by $p_t$ (initially equal to the number of players $p$). 

We define one more value: $u_t$, which we call the \textit{undisclosed hints}. This value counts the number of cards that can still retrieve a hint token which can be used before the deck is empty. We have:
\begin{equation}\label{utdef}
    u_t = \begin{cases}
    d_t - 1 & \text{ if } d_t > 0,\\
    0 & \text{ if } d_t = 0.
    \end{cases}
\end{equation}
Every card that is played or discarded can retrieve a hint token. This can be done $d_0$ times in total and then the deck is empty. However, if the last card that empties the deck retrieves a hint token, this token is only usable after the deck is empty. Thus, the value of $u_t$ is always one less than the current deck size $d_t$ (except when the deck is already empty).

Our $\Sigma_t$ is now defined as the sum over these previous values:
\begin{equation}\label{eq:sigmat}
    \Sigma_t = t + c_t + d_t + u_t + p_t
\end{equation}
and can be interpreted as the maximum possible number of total turns that is still reachable, after time step $t$.

\balance

We will now look into the effect of different actions on the values of $t$, $c_t$, $d_t$, $u_t$, $p_t$, and thus $\Sigma_t$.
A player can choose three actions in each turn: play, discard, or hint. The effect of each action on the different values is summarized in Table \ref{tab:action-effects}.

\begin{table}[]
\centering
\begin{threeparttable}[t]
\caption{Effect of actions on the different values.\\The three exceptions at the bottom have priority over the three standard actions at the top.}
\label{tab:action-effects}
\setlength\tabcolsep{4pt}%
\begin{tabular}{l|lllll|l}
action $a_t$ & $\Delta t$ & $\Delta c_t$ & $\Delta d_t$ & $\Delta u_t$ & $\Delta p_t$ & $\Delta\Sigma_t$ \\ \hline \rule{0pt}{2.3ex}%
play    & $+1$ & $0$   & $-1$  & $-1$  & $0$  & $-1$  \\
discard & $+1$ & $+1$  & $-1$  &  $-1$ & $0$  & $0$  \\
hint    & $+1$ & $-1$  &  $0$  &  $0$  & $0$  & $0$ \\ \hline \rule{0pt}{2.3ex}%
play a rank 5 successfully *
& $+1$ & $+1$ & $-1$ & $-1$ & $0$ & $0$ \\
$a_t$ empties the deck & $+1$ & $- c_{t-1}$ & $-1$  &  $0$ & $0$  & $\leq 0$  \\
$a_t$ while the deck is empty & $+1$ & $0$ & $0$  & $0$  & $-1$   &  $0$ \\
\end{tabular}
\begin{tablenotes}
* \footnotesize Only if $d_{t-1} > 1$ (otherwise it counts as an action that empties the deck or happens while the deck is empty) and $m_{t-1} < 8$ (otherwise it counts as a normal play action, since we do not gain a hint token with a rank 5 card if the hint budget is already full).
\end{tablenotes}
\end{threeparttable}
\end{table}

We see that the value of $\Sigma_t$ can never increase during a game. Furthermore, the values $t$, $d_t$, $p_t$, and $m_t$ must always stay non-negative according to the rules of Hanabi. This also implies that the values of $c_t$ and $u_t$ must always be non-negative, since $m_t$ and $d_t$ in \eqref{ctdef} and \eqref{utdef} are non-negative and integer. With this information, and from \eqref{eq:sigmat}, we can conclude that we must always have $t \leq \Sigma_t$.

Thus, the maximum value that $t$ could possibly reach is equal to the value of $\Sigma_0$ (before any action has been taken). We compute these starting values for every possible number of players $p$:
\vspace{-5pt}
\begin{table}[H]
    \centering
    \caption{Starting values of $\Sigma_t$.}
    \label{tab:sigma0}
\vspace{-12pt}
\begin{center}
    \begin{tabular}{c|cccc}
        $p$ & 2 & 3 & 4 & 5 \\ \hline \rule{0pt}{2.3ex}%
        $\Sigma_0$ & 89 & 80 & 79 & 72
    \end{tabular}
\end{center}
\end{table}
\vspace{-12pt}
As shown in Part 1, there is a particular sequence of actions in a Hanabi game, that gives the following outcome:
\begin{center}
    \begin{tabular}{l|lllll|l}
         & $t$ & $c_t$ & $d_t$ & $u_t$ & $p_t$ & $\Sigma_t$ \\ \hline \rule{0pt}{2.3ex}%
        start  & $0$ & $8$ & $40$ & $39$ & $2$ &  $89$ \\
        end & $89$ & $0$   & $0$ &  $0$ & $0$ & $89$  
    \end{tabular}
\end{center}
Therefore, the maximum length of a Hanabi game is 89 turns.

\end{proof}

As demonstrated in part 1 of the proof of Proposition~\ref{89} this maximum length can be reached if many cards are discarded and none are played. We are particularly interested in games where the algorithms perform well, i.e. score 25 points. The maximum length of a so-called \textit{perfect game} is 71. We found a different number (65) in the literature \cite{bad}, but this is incorrect.

\begin{proposition}\label{71}
The maximum length of a perfect Hanabi game is 71 turns.
\end{proposition}

\begin{proof}
This proof also consists of two parts.
First we will show that there exists a perfect Hanabi game of length 71. Part 2 proves that no perfect Hanabi game can have a higher number of turns than 71.

\textit{Part 1.}
Again, take a two-player Hanabi game. The initial deck size is 40. The players start out by spending their 8 hints. Then they play 22 cards successfully, finishing four fireworks. This gives them 4 extra hints, which they use immediately. The players now start a pattern of first discarding one card, and then giving one hint. This can be done 17 times. Then 1 card is played successfully that empties the deck. Both players have one more turn, in which they successfully play the rank 4 and 5 cards of the remaining firework. The number of turns is $8 + 22 + 4 + 17\cdot 2 + 1 + 2 = 71$.

\textit{Part 2.}
We use the same values as defined in part 2 of the proof of Proposition~\ref{89}. We will show that the value of $\Sigma_t$ must decrease to at most 71 for a game to finish in a perfect score.

To reach this score of 25, we need at least 25 play actions of course. In Table \ref{tab:action-effects} it is shown that every play action decreases $\Sigma_t$ by 1, aside from a few exceptions. These exceptions are:
\begin{enumerate}[  (1.)]
    \item Play a rank 5 card successfully when the number of hint tokens is less than 8 and the deck is not empty.
    \item Any action that empties the deck when the number of hint tokens is 0.
    \item Any action when the deck is empty.
\end{enumerate}
These exceptions can all be play actions that do not decrease the value of $\Sigma_t$. Let's try to keep $\Sigma_t$ as high as possible (as it represents the maximum number of turns we can reach) while still scoring 25 points. Thus, we need to make sure as many play actions as possible are classified as one of the three exceptions.

A perfect game can end before the deck is empty, on the deck-emptying move, or when it is already empty. Let us investigate the maximum number of exception play moves in all cases.

If the game ends,
\begin{itemize}
    \item before the deck is empty: we can use exception (1.) five times,
    \item on the deck-emptying move: we can use (1.) four times and (2.) once, 
    \item when the deck is empty: we can use (1.) four times,
    (2.) once, and (3.) $p$ number of times. Recall that $p$ stands for the number of players. 
\end{itemize}

From all these cases, we see that the maximum possible number of exception play moves is $5 + p$. In a two player game, this would mean that 7 play moves do not decrease $\Sigma_t$, while the other $25-7=18$ do. The maximum number of turns in that case is $89 - 18 = 71$. Recall that 89 is the starting value of $\Sigma_t$ in the two player case, see Table \ref{tab:sigma0}.

\begin{table}[H]
\centering
\caption{Maximum potential number of turns.}
\label{tab:perfect-p}
\vspace{-5pt}
\begin{tabular}{c|c}
\# players & maximum value of $\Sigma_t$ at end of perfect game \\ \hline \rule{0pt}{2.3ex}%
2   & $89 - (25 - (5 + 2)) = 71$                   \\
3   & $80 - (25 - (5 + 3)) = 63$                   \\
4   & $79 - (25 - (5 + 4)) = 63$                   \\
5   & $72 - (25 - (5 + 5)) = 57$                  
\end{tabular}
\end{table}
An overview of the maximum potential number of turns for different values of $p$ is shown in Table~\ref{tab:perfect-p}. We see that in the two player case this value is the highest, meaning that no perfect Hanabi game can possibly be longer than 71 turns. In Part 1 we have shown that a perfect game of this length is indeed possible. Therefore, the maximum length of a perfect Hanabi game is 71 turns. 

\end{proof}

%% file: sections/5-conclusion.tex
\section{Conclusion}
\label{sec:conclusion}

We have applied several actor-critic algorithms to Hanabi, a relatively new benchmark for collaborative multi-agent deep reinforcement learning. Using a simplified version of the game, we notice in our experiments that the Vanilla Policy Gradient (VPG) algorithm outperforms Proximal Policy Optimization (PPO) over multiple random seeds. In our analysis we see that although PPO learns quicker in the beginning, it eventually hits a plateau giving VPG the chance to surpass it. We hypothesize that PPO's clipping operation might be a reason for getting stuck in local minima.
Our small hyperparameter search is a limitation of this study, further research would be necessary to confirm the findings.


%% file: sections/a0-mdps.tex
\section{Overview of mathematical frameworks}
\label{sec:mdps}

In Table \ref{tab:mdp-types} we provide an overview of possible decision processes to work with in reinforcement learning (RL). The simplified version of Hanabi we used in the paper falls in the category of a Dec-MDP. Notice that we have also made a distinction in collaborative games (where all agents receive identical rewards) and adversarial games (different rewards).



\begin{table}[H]
\caption{Mathematical frameworks for RL (with examples of games in parentheses).}
\label{tab:mdp-types}
\centering
\setlength\tabcolsep{4pt}%
\begin{tabular}{l|l|l}
    & fully observable  & partially observable  \\ \hline \rule{0pt}{2.4ex}%
single-agent         & \begin{tabular}[t]{@{}l@{}} MDP \\ (PacMan, Tetris) \end{tabular}     & \begin{tabular}[t]{@{}l@{}} POMDP \\ (Minesweeper) \end{tabular}          \\ \hline \rule{0pt}{2.4ex}%
\begin{tabular}[t]{@{}l@{}} multi-agent \\  identical rewards\end{tabular} & \begin{tabular}[t]{@{}l@{}} Dec-MDP or MMDP \\ (Pandemic) \end{tabular}                                & \begin{tabular}[t]{@{}l@{}} Dec-POMDP \\ (Hanabi) \end{tabular}                  \\  \hline \rule{0pt}{2.4ex}%
\begin{tabular}[t]{@{}l@{}} multi-agent \\ different rewards\end{tabular}  & \begin{tabular}[t]{@{}l@{}} SG\tablefootnote{Stochastic game. Despite the name, it does not necessarily have to be stochastic.} \\ (Chess, Go) \end{tabular} & \begin{tabular}[t]{@{}l@{}} POSG\tablefootnote{Partially observable stochastic game.} \\ (Poker) \end{tabular}
\end{tabular}
\end{table}


%% file: sections/c-hyperparams.tex
\section{Algorithm design and hyperparameters}
\label{sec:hyperparams}

All the options shown in Table~\ref{tab:hyperparams} were selected through a search of many short, preliminary experiments. These are the settings of the algorithms presented in the paper.

\begin{table}[H]
    \centering
    \caption{Settings of the three algorithms.}
    \label{tab:hyperparams}
    \setlength\tabcolsep{4pt}%
    \begin{tabular}{l|l|l|l}
        Algorithm & \textbf{SPG} & \textbf{VPG} & \textbf{PPO}\\ \hline \rule{0pt}{2.3ex}%
        \textbf{Network arch.}  & & & \\
        Hidden layers $\pi$ & [128,128,64] & [128,128,64] & [128,128,64] \\
        Hidden layers $V$ & - & [128,64,32] & [128,64,32] \\
        Activation func. $\pi$ & Tanh & Tanh & Tanh \\
        Activation func. $V$ & - & Tanh & Tanh \\
        \textbf{Representations} & & & \\
        State (input $\pi$ and $V$) & 136 & 136 & 136  \\
        Action (output $\pi$) & 11 & 11 & 11\\ 
        \textbf{Rewards} & & & \\
        Successful play & $+10$ & $+10$ & $+10$ \\
        Lost all lives & $-$score & $-$score & $-$score \\
        Illegal move & $-1$ & $-1$ & $-1$\\
        Lost one life & $-$0.1 & $-$0.1 & $-$0.1 \\
        Hint & $-$0.02 & $-$0.02 & $-$0.02 \\
        Play & $+$0.02 & $+$0.02 & $+$0.02\\
        Discard playable & $-$0.1 & $-$0.1 & $-$0.1 \\ 
        Discard useless & $+$0.1 & $+$0.1 & $+$0.1 \\
        Discard unique & $-$0.1 & $-$0.1 & $-$0.1 \\
        \textbf{Objective} & & & \\
        Advantage type & - & GAE & GAE \\
        GAE parameter $(\lambda)$ & - & 0.95 & 0.95 \\
        Clipping parameter $(\varepsilon)$ & - & - & 0.2 \\
        Entropy coefficient $(\beta)$ & $0.01$ & $0.01$ & $0.01$ \\
        \textbf{Optimizer} &  &  &  \\
        Learning rate $\pi$ & $3\cdot 10^{-4}$ & $3\cdot 10^{-4}$ & $3\cdot 10^{-4}$ \\
        Learning rate $V$ & - & $3\cdot 10^{-4}$ & $3\cdot 10^{-4}$ \\
        \textbf{Hyperparameters} & & &\\
        Batch size & $1000$ & $1000$ & $1000$ \\
        Renormalize $G_t$ or $A^{\pi_{\bm{\theta}}}$ & yes & yes & yes \\
        Discount factor $(\gamma)$ & $0.99$ & $0.99$ & $0.99$\\
        Update iterations $\pi$ & 1 & 1 & 5 \\
        Update iterations $V$ & - & 5 & 5 \\
    \end{tabular}
\end{table}

%% file: sections/e-extra-figs.tex
\section{extra figures}
\label{sec:extra-figs}


\begin{figure}[H]
    \centering
    \includegraphics[width=0.9\linewidth]{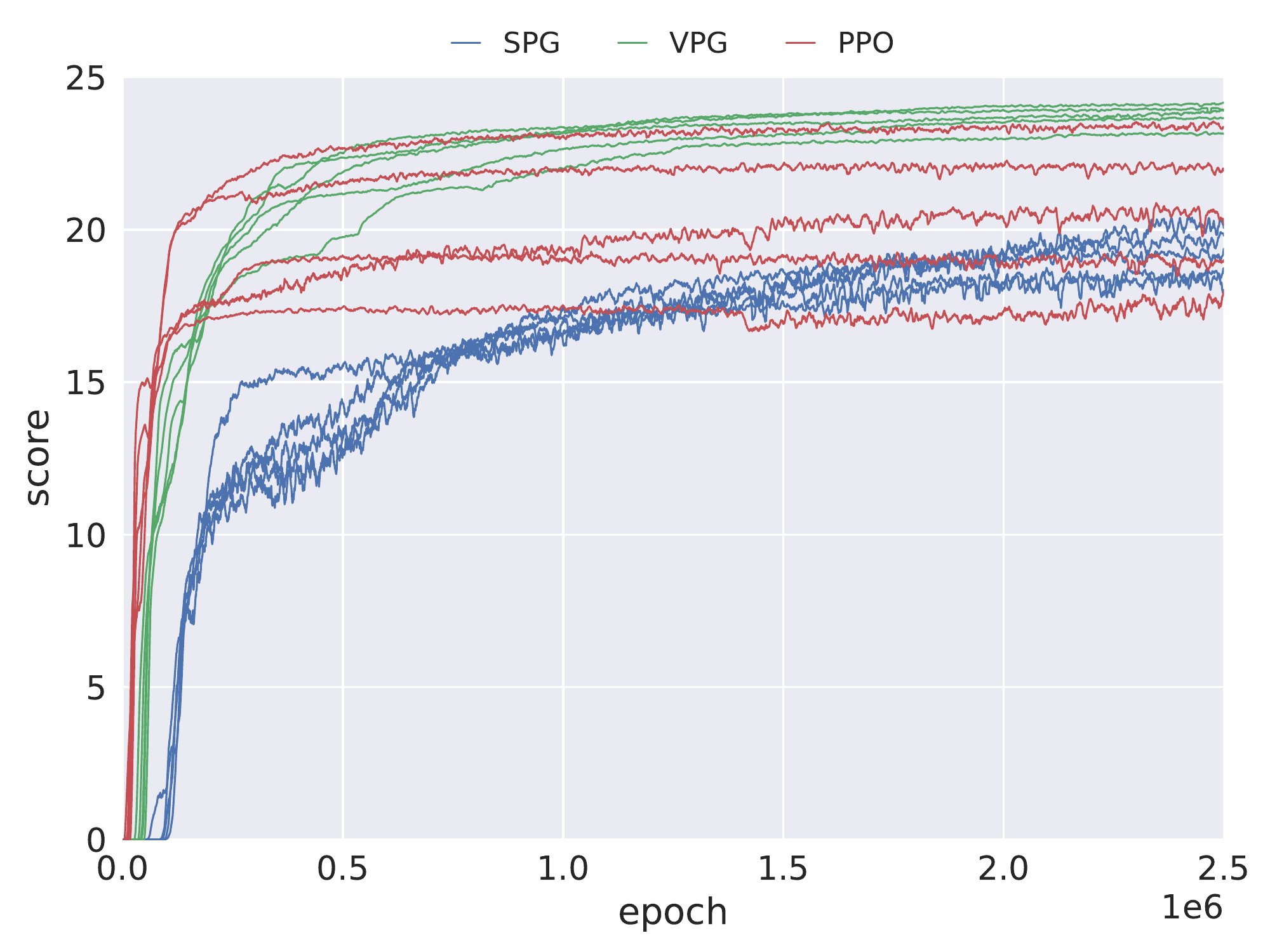}
    \caption{Complete learning curves shown separately for every random seed. PPO hits a plateau at varying levels.}
    \label{fig:separate-plot}
\end{figure}

\begin{figure}[H]
    \centering
    \includegraphics[width=0.9\linewidth]{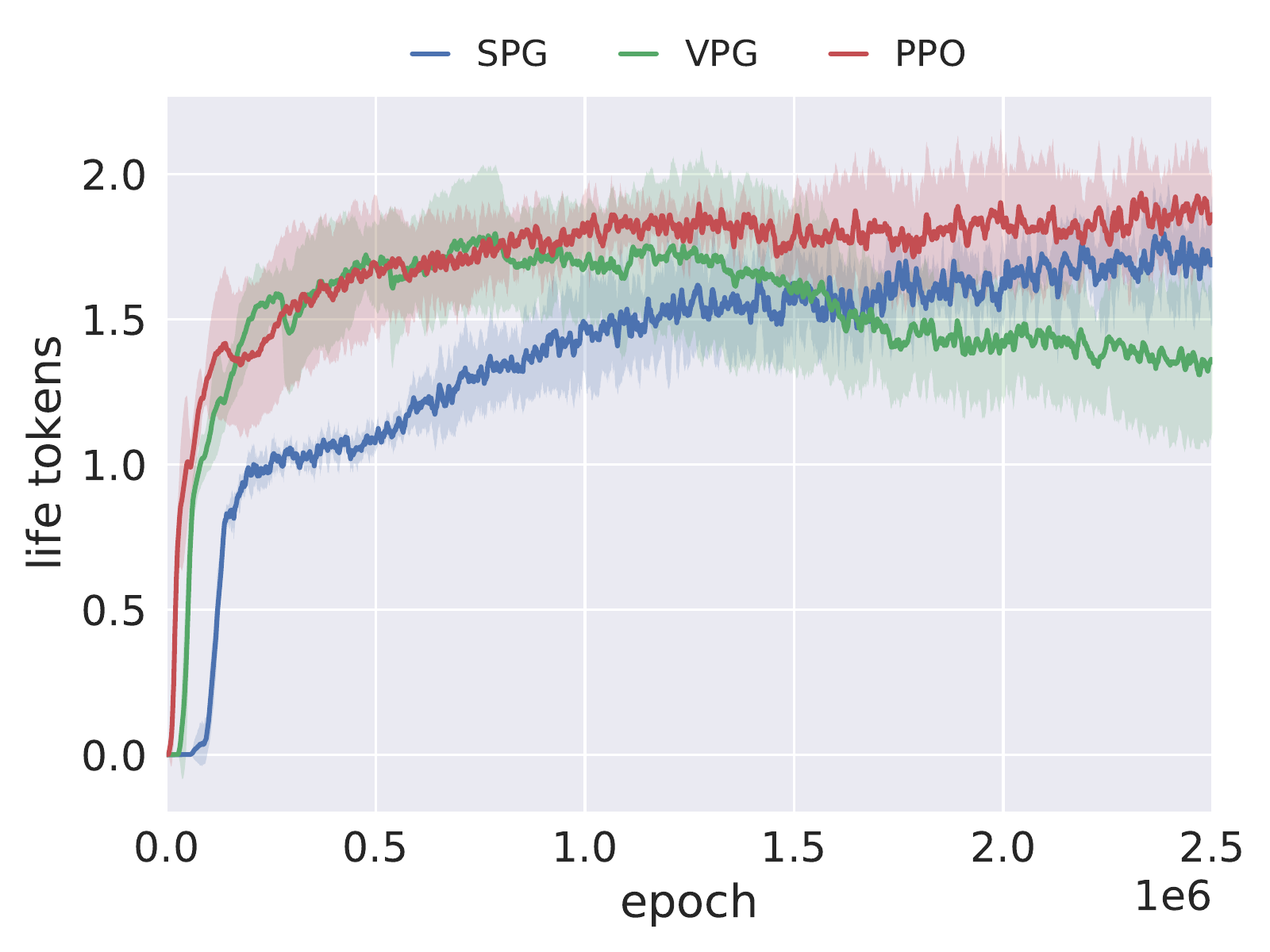}
    \caption{Average number of life tokens left at the end of an episode.}
    \label{fig:life-max}
\end{figure}

\begin{figure}[H]
    \centering
    \includegraphics[width=0.9\linewidth]{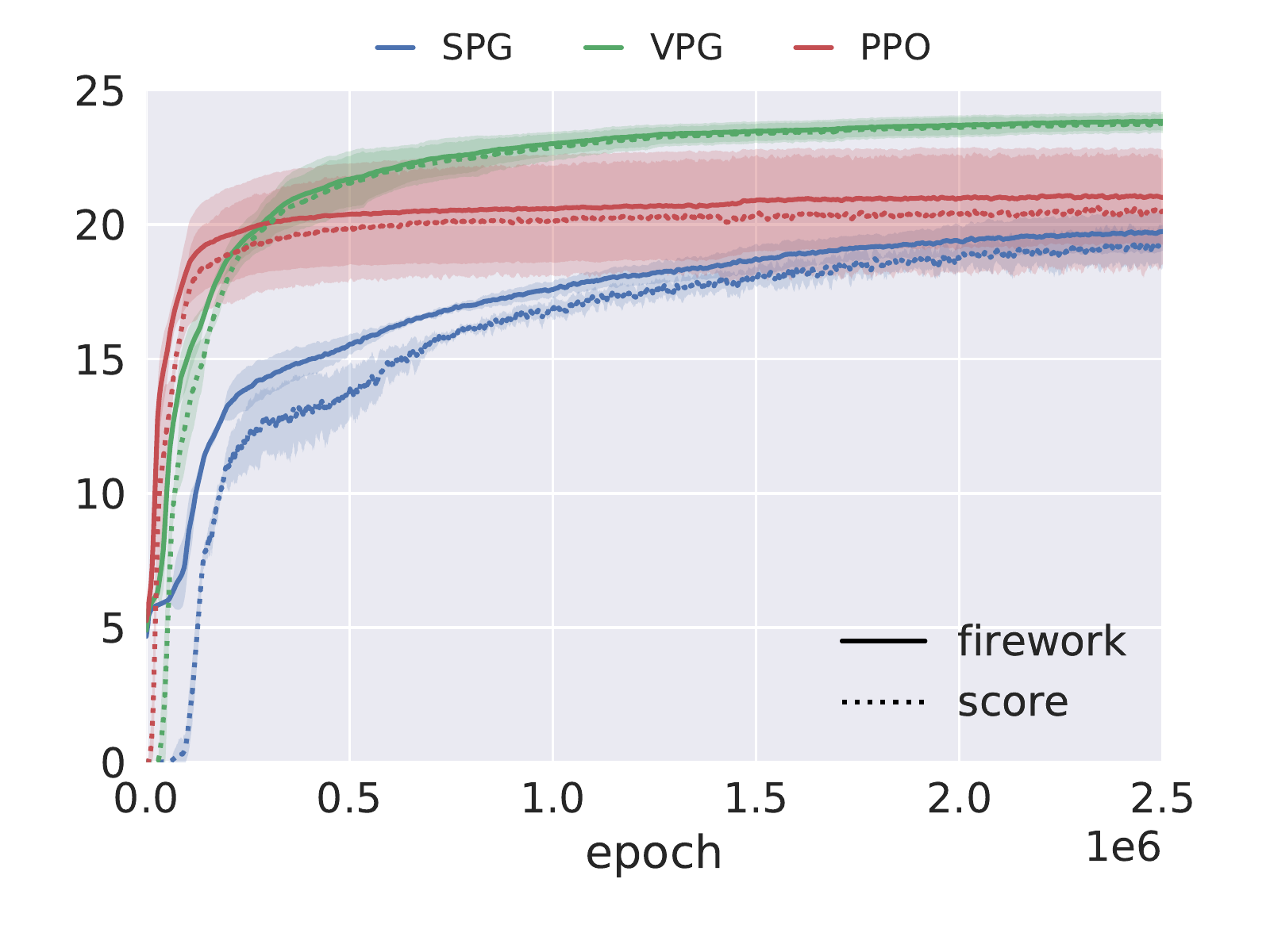}
    \caption{Development of fireworks and scores during training.}
    \label{fig:fireworks-max}
\end{figure}

\begin{figure}[H]
    \centering
    \includegraphics[width=0.8\linewidth]{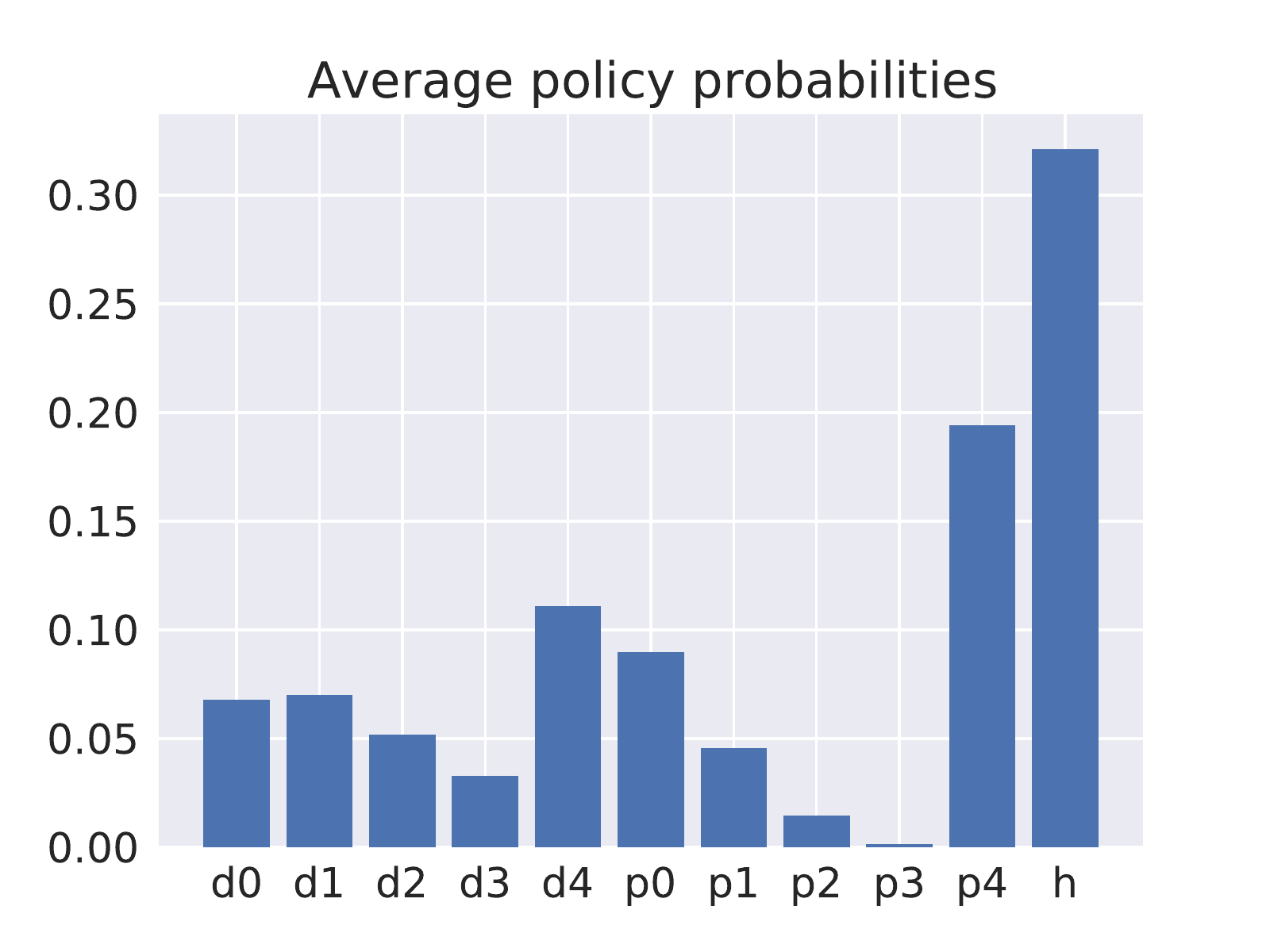}
    \caption{Average action selection probabilities of one SPG run during the last epoch.}
    \label{fig:probs-histo-spg}
\end{figure}

\begin{figure}[H]
    \centering
    \includegraphics[width=0.8\linewidth]{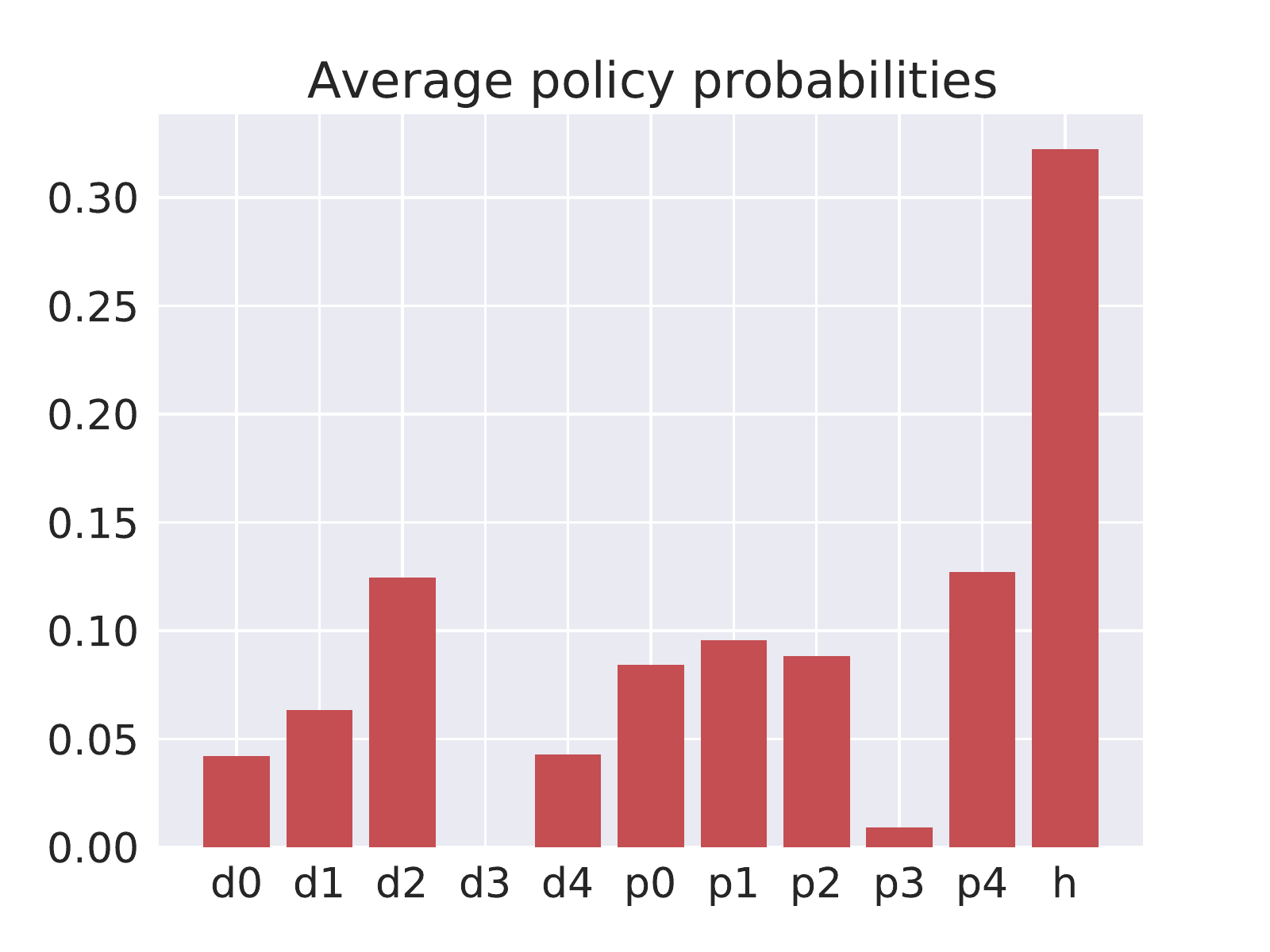}
    \caption{Average action selection probabilities of one PPO run during the last epoch.}
    \label{fig:probs-histo-ppo}
\end{figure}

\begin{figure}[H]
    \centering
    \includegraphics[width=0.85\linewidth]{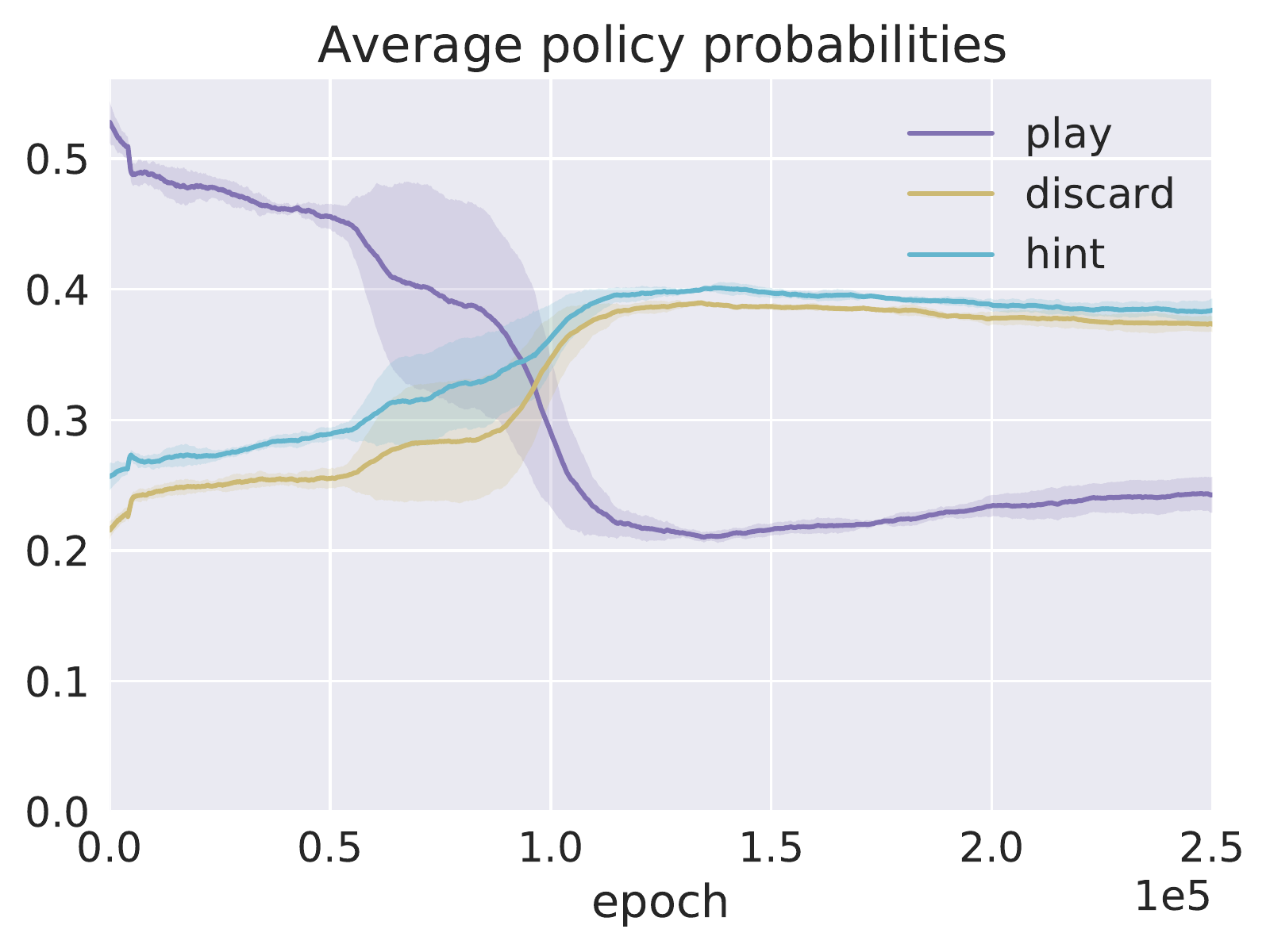}
    \caption{Average policy of our SPG agents during the first 10\% of training.}
    \label{fig:actions-spg}
\end{figure}

\begin{figure}[H]
    \centering
    \includegraphics[width=0.85\linewidth]{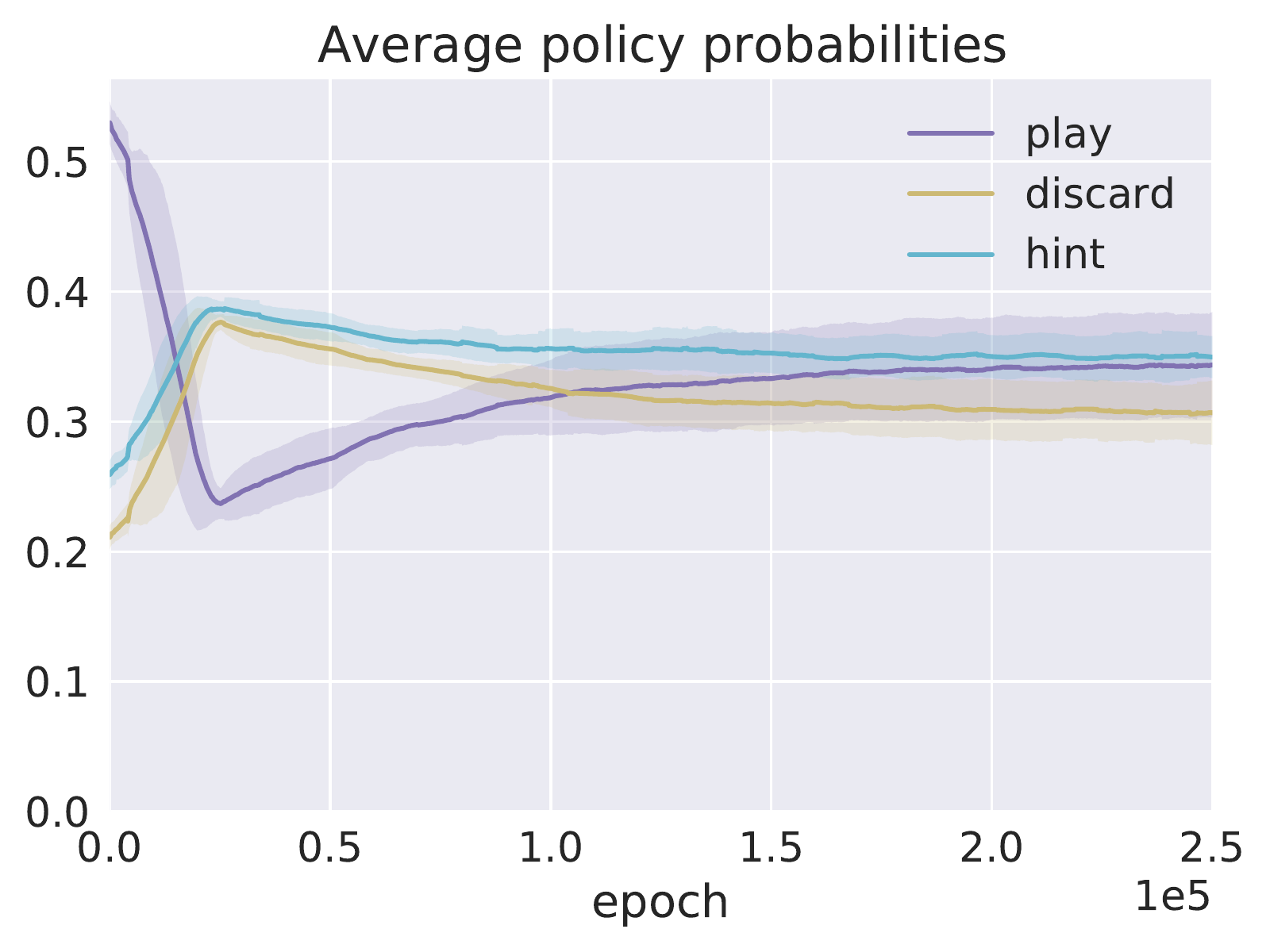}
    \caption{Average policy of our PPO agents during the first 10\% of training.}
    \label{fig:actions-ppo}
\end{figure}

\begin{figure}[H]
    \centering
    \includegraphics[width=0.85\linewidth]{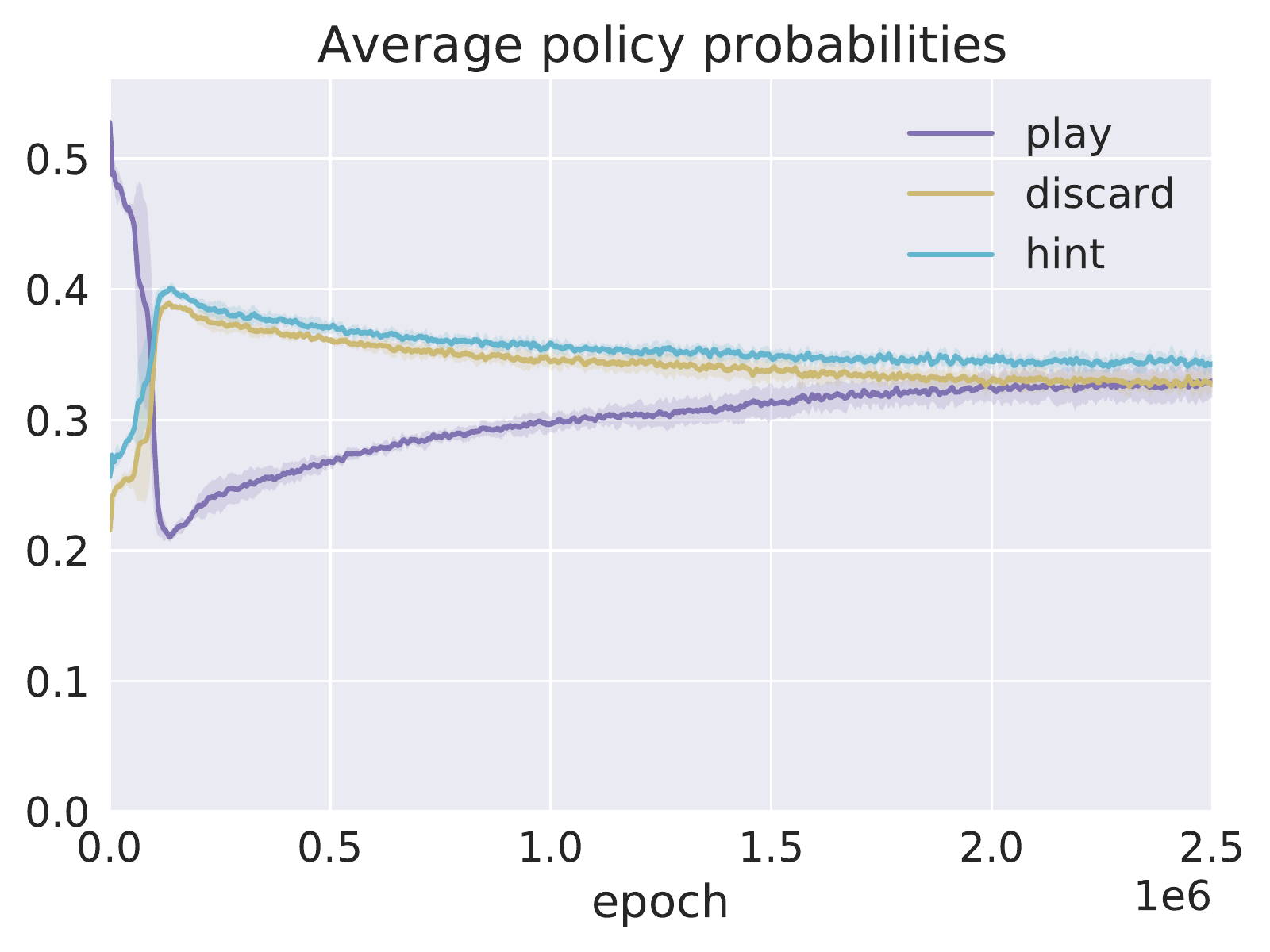}
    \caption{Average policy of our SPG agents during all of training.}
    \label{fig:actions-spg-max}
\end{figure}

\begin{figure}[H]
    \centering
    \includegraphics[width=0.85\linewidth]{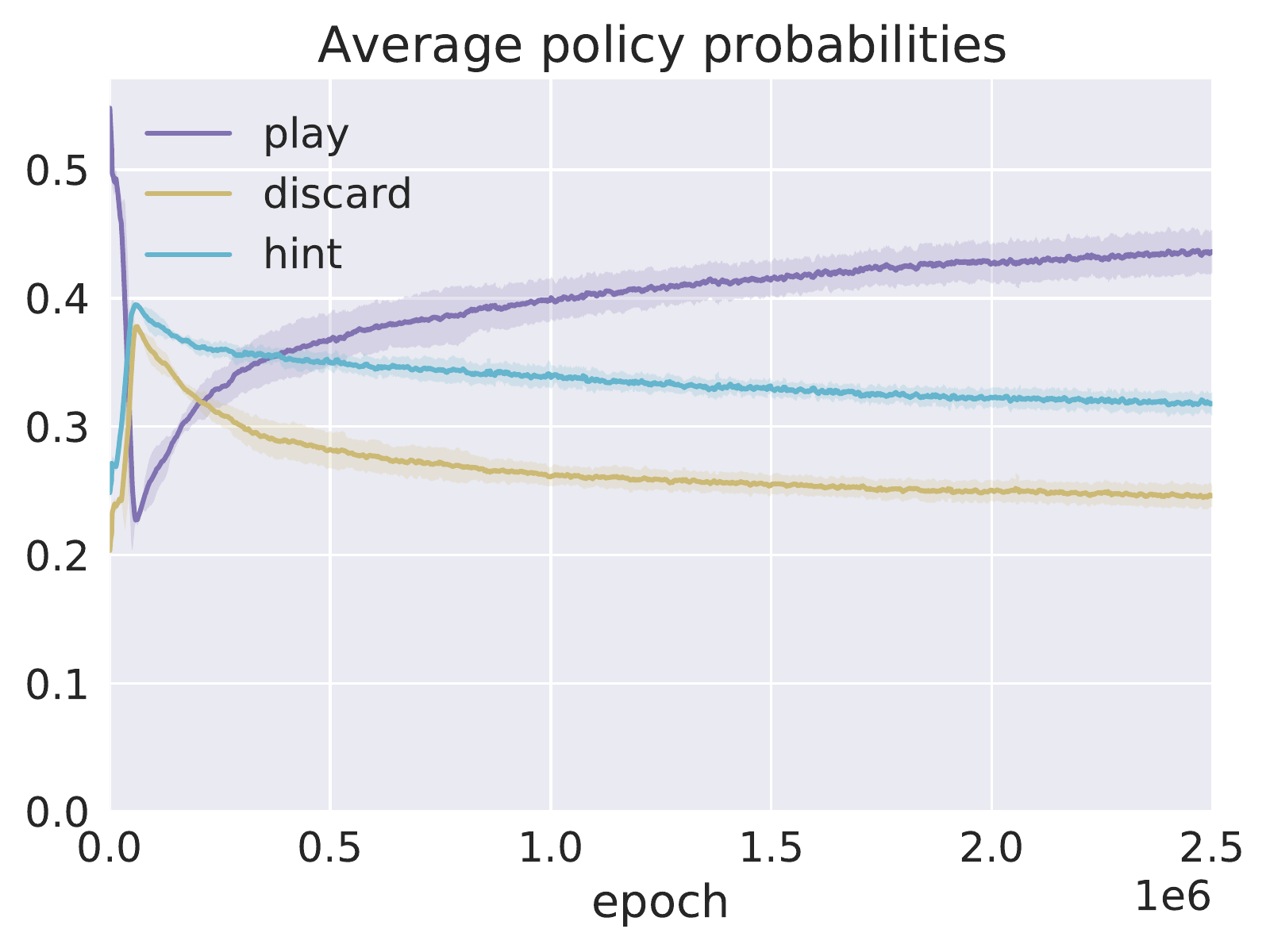}
    \caption{Average policy of our VPG agents during all of training.}
    \label{fig:actions-vpg-max}
\end{figure}

\begin{figure}[H]
    \centering
    \includegraphics[width=0.85\linewidth]{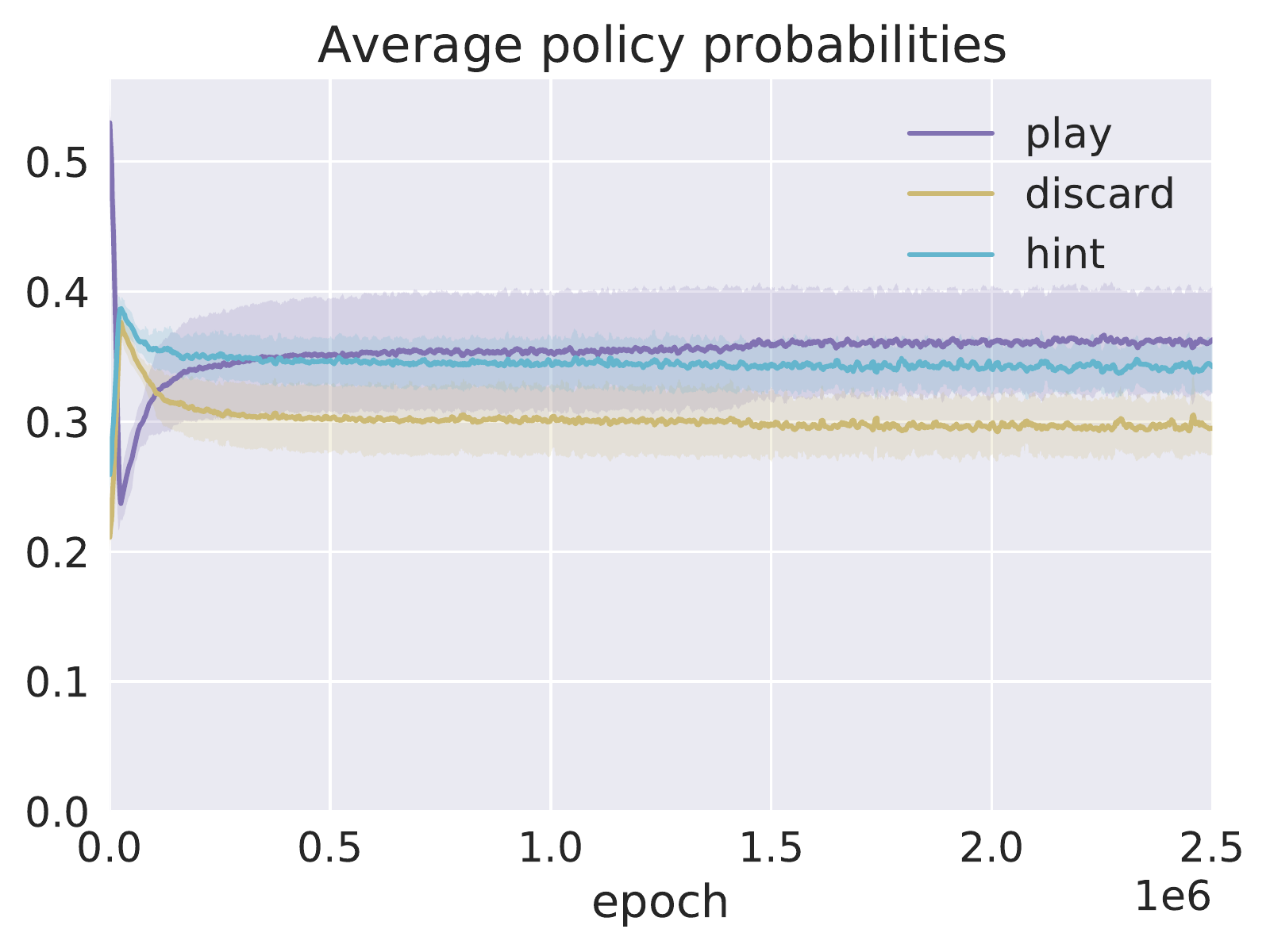}
    \caption{Average policy of our PPO agents during all of training.}
    \label{fig:actions-ppo-max}
\end{figure}

\begin{figure}[H]
    \centering
    \includegraphics[width=0.85\linewidth]{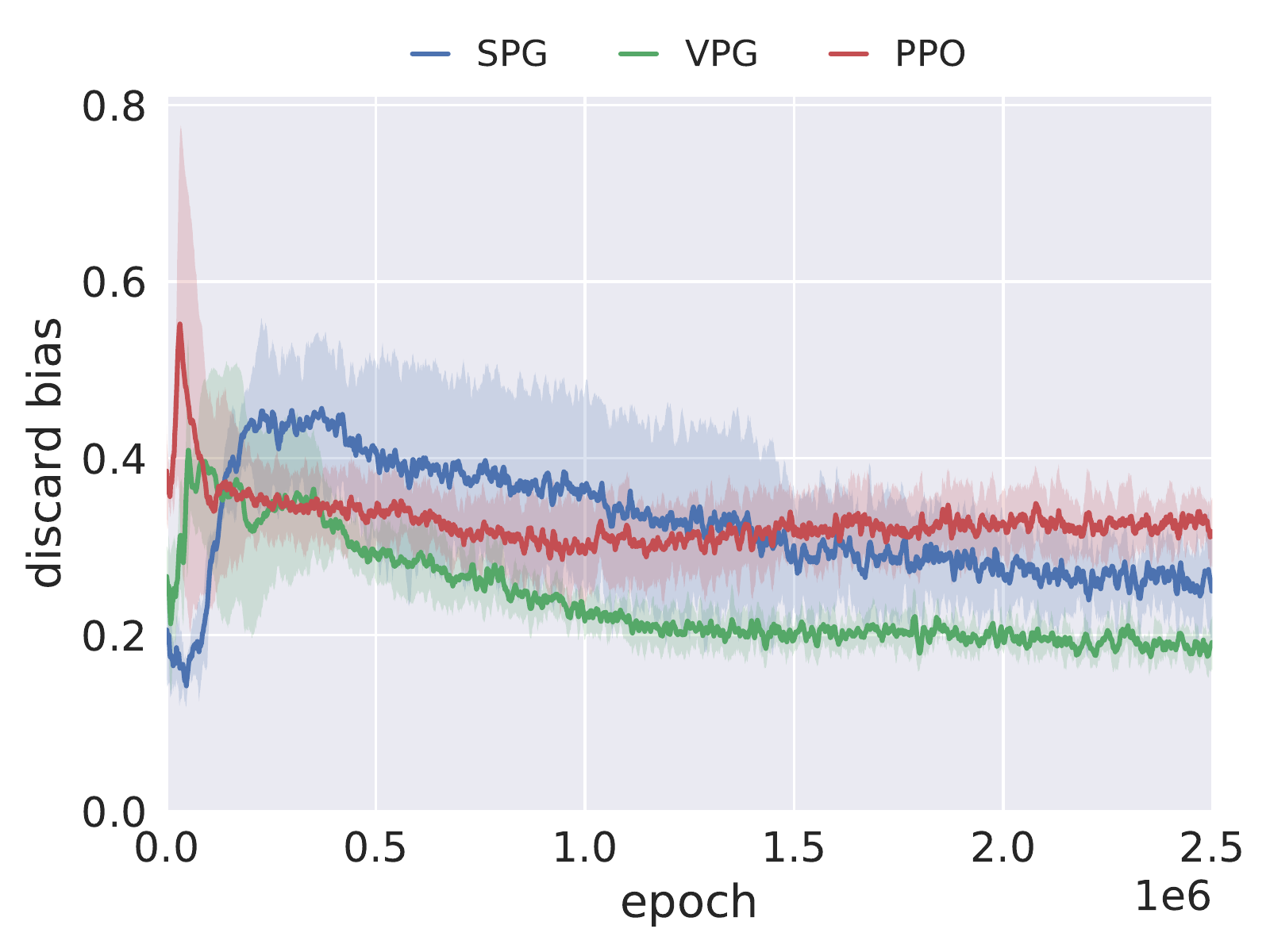}
    \caption{Positional bias of the discard actions during training.}
    \label{fig:bias-discard}
\end{figure}